\documentclass{article} 
\usepackage{amssymb,amsmath,amsthm,bbm}
\usepackage{verbatim,float,url,dsfont}
\usepackage{graphicx,subfigure,psfrag}
\usepackage{algorithm,algorithmic}
\usepackage{mathtools,enumitem}
\usepackage{multirow}
\usepackage{ragged2e}
\usepackage{xr-hyper}
\usepackage{array}
\usepackage{stmaryrd,scalerel}

\usepackage[colorlinks=true,citecolor=blue,urlcolor=blue,linkcolor=blue]{hyperref}
\usepackage[margin=1in]{geometry}

\usepackage[utf8]{inputenc} 
\usepackage[T1]{fontenc}    
\usepackage{booktabs}       
\usepackage{nicefrac}         
\usepackage{microtype}      
\usepackage{multicol}

\ifdefined\TimesFont 
\usepackage{times} 
\fi

\ifdefined\ParSkip 
\usepackage{parskip} 
\fi

\usepackage{mathtools} 
\mathtoolsset{showonlyrefs}

\newtheorem{theorem}{Theorem}

\newtheorem{proposition}{Proposition}
\theoremstyle{definition}

\newtheorem*{assumption*}{\assumptionnumber}
\providecommand{\assumptionnumber}{}


\makeatletter
\newcommand*\rel@kern[1]{\kern#1\dimexpr\macc@kerna}
\newcommand*\widebar[1]{%
  \begingroup
  \def\mathaccent##1##2{%
    \rel@kern{0.8}%
    \overline{\rel@kern{-0.8}\macc@nucleus\rel@kern{0.2}}%
    \rel@kern{-0.2}%
  }%
  \macc@depth\@ne
  \let\math@bgroup\@empty \let\math@egroup\macc@set@skewchar
  \mathsurround\z@ \frozen@everymath{\mathgroup\macc@group\relax}%
  \macc@set@skewchar\relax
  \let\mathaccentV\macc@nested@a
  \macc@nested@a\relax111{#1}%
  \endgroup
}
\makeatother

\newcommand{\minimize}{\mathop{\mathrm{minimize}}}

\def\R{\mathbb{R}}

\def\N{\mathbb{N}}
\def\E{\mathbb{E}}

\def\sign{\mathrm{sign}}

\def\cC{\mathcal{C}}

\def\cX{\mathcal{X}}
\def\cY{\mathcal{Y}}

\def\err{\mathrm{err}}

\def\ind#1{\mathbbm{1}\left\{#1\right\}}
\newcommand{\quantile}[2]{\mathrm{Quantile}_{#1}\left(#2\right)}

\title{Conformal PID Control for Time Series Prediction}  

\author{Anastasios N. Angelopoulos$^{1}$, Emmanuel J. Cand\`es$^{2}$, Ryan
  J. Tibshirani$^{1}$} 
\date{$^1$University of California, Berkeley, $^2$Stanford University \\
\texttt{\{angelopoulos,ryantibs\}@berkeley.edu}, \texttt{candes@stanford.edu}} 

\begin{document}
\maketitle

\begin{abstract}
We study the problem of uncertainty quantification for time series prediction,
with the goal of providing easy-to-use  algorithms with formal
guarantees. The algorithms we present build upon ideas from conformal prediction
and control theory, are able to prospectively model conformal scores in an
online setting, and adapt to the presence of systematic errors due to
seasonality, trends, and general distribution shifts. Our theory both simplifies
and strengthens existing analyses in online conformal prediction.  
Experiments on 4-week-ahead forecasting of statewide COVID-19 death counts in
the U.S.\ show an improvement in coverage over the ensemble forecaster used in
official CDC communications. We also run experiments on predicting electricity
demand, market returns, and temperature using autoregressive, Theta, Prophet,
and Transformer models. We provide an extendable codebase for testing our
methods and for the integration of new algorithms, data sets, and forecasting 
rules.\footnote{\url{http://github.com/aangelopoulos/conformal-time-series}}    
\end{abstract}

\section{Introduction}
Machine learning models run in production systems regularly encounter data
distributions that change over time. This can be due to factors such as
seasonality and time-of-day, continual updating and re-training of upstream
machine learning models, changing user behaviors, and so on. These distribution
shifts can degrade a model's predictive performance. They also invalidate
standard techniques for uncertainty quantification, such as \emph{conformal
  prediction} \cite{vovk1999machine, vovk2005algorithmic}.  

To address the problem of shifting distributions, we can consider the task of prediction in an adversarial online setting, as in \cite{gibbs2021adaptive}. 
In this setting, we observe a (potentially)
adversarial time series of deterministic covariates $x_t \in \cX$ and responses
$y_t \in \cY$, for $t \in \N = \{1,2,3,\dots\}$. As in standard conformal
prediction, we are free to define any \emph{conformal score function} $s_t : \cX
\times \cY \to \R$, which we can view as measuring the accuracy of our
forecast at time $t$. We will assume with a loss of generality that $s_t$ is 
negatively oriented (lower values mean greater forecast accuracy). For example,
we may use the absolute error $s_t(x, y) = |y - f_t(x)|$, where $f_t$ is a
forecaster trained on data up to but not including data at time $t$.

The challenge in the sequential setting is as follows. We seek to invert the 
score function to construct a \emph{conformal prediction set},  
\begin{equation}
\label{eq:conformal-set}
\cC_t = \{y \in \cY: s_t(x_t, y) \leq q_t\},
\end{equation}
where $q_t$ is an estimated $1-\alpha$ quantile for the score $s_t(x_t,y_t)$
at time $t$. In standard conformal prediction, we would take $q_t$ to be a level
$1-\alpha$ sample quantile (up to a finite-sample correction) of $s_t(x_i,y_i)$,
$i < t$; if the data sequence $(x_i, y_i)$, $i \in \N$ were i.i.d.\ or
exchangeable, then this would yield $1-\alpha$ coverage
\cite{vovk2005algorithmic} at each time $t$. However, in the sequential setting,
which does not assume exchangeability (or any probabilistic model for the data 
for that matter), choosing $q_t$ in \eqref{eq:conformal-set} to yield coverage
is a formidable task. In fact, if we are not willing to make any assumptions
about the data sequence, then achieving coverage at time $t$ would require the 
user to construct prediction intervals of infinite sizes.

Therefore, our goal is to achieve \emph{long-run coverage} in time. That is,
letting $\err_t = \ind{y_t \notin \cC_t}$, we would like to achieve, for large
integers $T$,
\begin{equation}
\label{eq:long-run-coverage}
\frac{1}{T} \sum_{t=1}^T \err_t = \alpha + o(1)
\end{equation}
under few or no assumptions, where $o(1)$ denotes a quantity that tends to zero
as $T \to \infty$. We note that \eqref{eq:long-run-coverage} is not probabilistic
at all, and every theoretical statement we will make in this paper holds
deterministically. Furthermore, going beyond \eqref{eq:long-run-coverage}, we
also seek to design flexible strategies to produce the sharpest prediction sets 
possible, which not only adapt to, but also anticipate distribution shifts.   

We call our proposed solution \emph{conformal PID control}. It treats the system for producing prediction sets as a proportional-integral-derivative (PID)
controller. In the language of control, the prediction sets take a \emph{process
  variable}, $q_t$, and then produce an output, $\err_t$. We seek to anchor
$\err_t$ to a \emph{set point}, $\alpha$. To do so, we apply corrections to
$q_t$ based on the error of the output, $g_t = \err_t - \alpha$. By reframing
the problem in this language, we are able to build algorithms that have more
stable coverage while also prospectively adapting to changes in the score
sequence, much in the same style as a control system. See the diagram in
Figure \ref{fig:control-simplified}. 

\begin{figure}[h]
\centering
\includegraphics[width=0.75\textwidth]{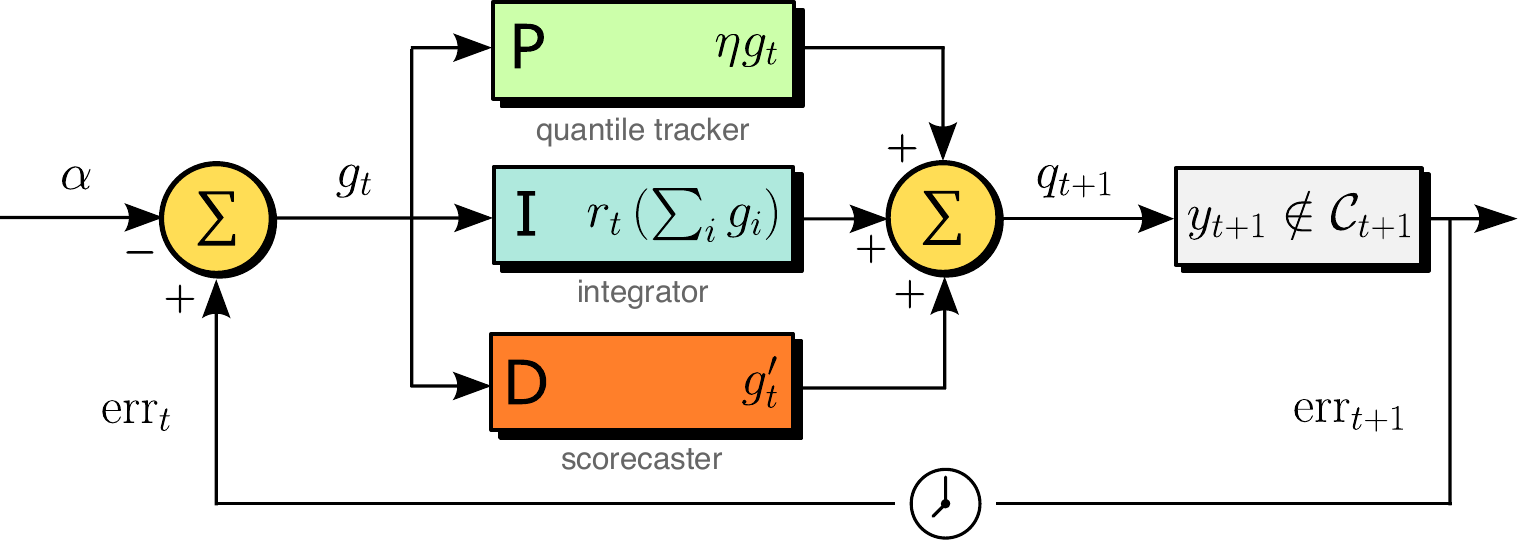}
\caption{Conformal PID Control, expressed as a block diagram.} 
\label{fig:control-simplified}
\end{figure}

\subsection{Peek at results: methods}

Three design principles underlie our methods: 

\begin{enumerate}
\item \emph{Quantile tracking (P control).} Running online gradient descent
  on the quantile loss (summed over all past scores) gives rise to a method that
  we call \emph{quantile tracking}, which achieves long-run coverage
  \eqref{eq:long-run-coverage} under no assumptions except boundedness of the
  scores. This bound can be unknown. Unlike adaptive conformal inference (ACI) 
  \cite{gibbs2021adaptive}, quantile tracking does not return infinite sets after a sequence of
  miscoverage events. This can be seen as equivalent to proportional (P)
  control.  

\item \emph{Error integration (I control).} By incorporating the running sum
  \smash{$\sum_{i=1}^t (\err_i - \alpha)$} of the coverage errors into the
  online quantile updates, we can further stabilize the coverage. This
  \emph{error integration} scheme achieves long-run coverage
  \eqref{eq:long-run-coverage} under no assumptions whatsoever on the  
  scores (they can be unbounded). This can be seen as equivalent to integral 
  (I) control. 

  \item \emph{Scorecasting (D control).} To account for systematic trends in the
    scores---this may be due to aspects of the data distribution, fixed or
    changing, which are not captured by the initial forecaster---we train a
    second model, namely, a \emph{scorecaster}, to predict the quantile of the 
    next score. While quantile tracking and error integration are merely
    reactive, scorecasting is forward-looking. It can potentially residualize 
    out systematic trends in the errors and lead to practical advantages 
    in terms of coverage and efficiency (set sizes). This can be seen as
    equivalent to derivative (D) control. Traditional control theory would
    suggest using a linear approximation $g'_t = g_t - g_{t-1}$, but in our
    problem, we will typically choose more advanced scorecasting algorithms
    that go well beyond simple difference schemes. 
\end{enumerate}

These three modules combine to make our final iteration, the \emph{conformal PID
  controller}: 
\begin{equation}
\label{eq:pid-update}
q_{t+1} = \underbrace{\vphantom{\Bigg( \sum_{i=1}^t} \eta g_t}_{\text{P}} +\, 
\underbrace{r_t\Bigg( \sum_{i=1}^t g_t \Bigg)}_{\text{I}} \,+
\underbrace{\vphantom{\Bigg( \sum_{i=1}^t} g'_t}_{\text{D}}. 
\end{equation}
In traditional PID control, one would take $r_t(x)$ to be a linear function of
$x$. Here, we allow for nonlinearity and take $r_t$ to be a \emph{saturation  
function} obeying   
\begin{equation}
\label{eq:saturation}
x \geq c \cdot h(t) \implies r_t(x) \geq b, \;\;\; \text{and} \;\;\;
x \leq -c \cdot h(t) \implies r_t(x) \leq -b,
\end{equation}
for constants $b,c>0$, and a sublinear, nonnegative, nondecreasing function
$h$---we call a function $h$ satisfying these conditions \emph{admissible}. An
example is the \emph{tangent integrator} \smash{$r_t(x) = K_{\text{I}} \tan(x
  \log(t) /(t C_{\text{sat}}))$}, where we set $\tan(x) = \sign(x) \cdot
\infty$ for $x \notin [-\pi/2, \pi/2]$, and \smash{$C_{\text{sat}},
  K_{\text{I}}>0$} are constants. The choice of integrator $r_t$ is a design
decision for the user, as is the choice of scorecaster $g'_t$.    

We find it convenient to reparametrize \eqref{eq:pid-update}, to produce a
sequence of quantile estimates $q_t$, $t \in \N$ used in the prediction sets
\eqref{eq:conformal-set}, as follows: 
\begin{equation}
\begin{aligned}
\label{eq:main}
\text{let $\hat{q}_{t+1}$ be any function of the past: $x_i,y_i,q_i$, for $i
  \leq t$}, \\  
\text{then update } q_{t+1} = \hat{q}_{t+1} + r_t \Bigg( \sum_{i=1}^t (\err_i -   
\alpha) \Bigg).  
\end{aligned}
\end{equation}
Taking \smash{$\hat{q}_{t+1} = \eta g_t + g'_t$} recovers \eqref{eq:pid-update},  
but we find it generally useful to instead consider the formulation in
\eqref{eq:main}, which will be our main focus in the
exposition henceforth. Now we view \smash{$\hat{q}_{t+1}$} as the scorecaster,
which directly predicts $q_{t+1}$ using past data. A main result of this paper, 
whose proof is given in Section \ref{sec:methods}, is that the conformal PID 
controller \eqref{eq:main} yields long-run coverage for any
choice of integrator $r_t$ that satisfies the appropriate saturation condition,
and any scorecaster \smash{$\hat{q}_{t+1}$}.       

\begin{theorem}
\label{thm:main}
Let \smash{$\{\hat{q}_t\}_{t \in \N}$} be any sequence of numbers in $[-b/2, 
b/2]$ and let \smash{$\{s_t\}_{t \in \N}$} be any sequence of score functions
with outputs in $[-b/2, b/2]$. Here $b > 0$, and may be infinite. Assume that
$r_t$ satisfies \eqref{eq:saturation}, for an admissible function $h$. Then the  
iterations in \eqref{eq:main} achieve long-run coverage
\eqref{eq:long-run-coverage}.   
\end{theorem}

To emphasize, this result holds deterministically, with no probabilistic model 
on the data $(x_t, y_t)$, $t \in \N$. (Thus in the case that the sequence is
random, the result holds for all realizations of the random variables.) As we
will soon see, this theorem can be seen as a generalization of existing results
in the online conformal literature. 

\subsection{Peek at results: experiments}

\paragraph{COVID-19 death forecasting.}

To demonstrate conformal PID in practice, we consider 4-week-ahead forecasting
of COVID-19 deaths in California, from late 2020 through late 2022. The base
forecaster $f_t$ that we use is the ensemble model from the COVID-19 Forecast
Hub, which is the model used for official CDC communications on COVID-19
forecasting \cite{cramer2022united, ray2023comparing}. 
In this
forecasting problem, at each time $t$ we actually seek to predict the observed
death count $y_{t+4}$ at time $t+4$.  

\begin{figure}[ht]
\includegraphics[width=\textwidth]{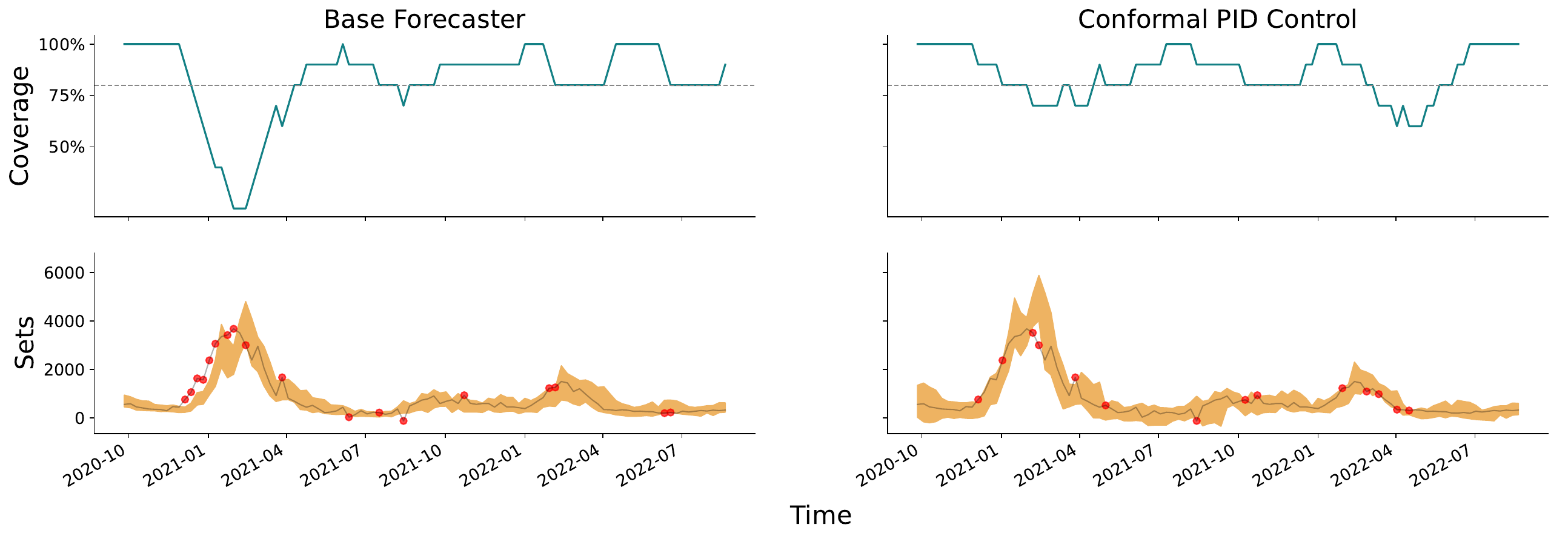}
\caption{Results for 4-week ahead COVID-19 death forecasting in California. The
  left column shows the COVID-19 Forecast Hub ensemble model, and the right
  column shows conformal PID control using the tan integrator, and a scorecaster 
  given by $\ell_1$-penalized quantile regression on all past forecasts, cases,
  and deaths from all 50 states. The top row plots coverage, averaged over a
  trailing window of 10 weeks. The nominal coverage level is $1-\alpha = 0.8$
  and marked by a gray dotted line. The bottom row plots the prediction sets in
  gold along with the ground-truth times series (death counts). Miscoverage
  events are indicated by red dots. Summary statistics such as the coverage and
  average set size are available in Table \ref{fig:table-teaser-1}.}
\label{fig:teaser-1}
\end{figure}

Figure \ref{fig:teaser-1} shows the central 80\% prediction sets from the
Forecast Hub ensemble model on the left panel, and those from our conformal PID
method on the right. We use a quantile conformal score function, as in
conformalized quantile regression \cite{romano2019conformalized}, applied
asymmetrically (i.e., separately) to the lower and upper quantile levels). We use 
the tan integrator, with constants chosen heuristically (as described in
Appendix \ref{app:heuristics}), and an $\ell_1$-regularized quantile regression 
as the scorecaster---in particular, the scorecasting model at time $t$ predicts 
the quantile of the score at time $t+4$ based on all previous forecasts, cases,
and deaths, from \emph{all 50 US states}. The main takeaway is that conformal
PID control is able to correct for consistent underprediction of deaths in the
winter wave of late 2020/early 2021. We can see from the figure that the
original ensemble fails to cover 8 times in a stretch of 10 weeks, resulting in
a coverage of 20\%; meanwhile, conformal PID only fails to cover 3 times during 
this stretch, restoring the coverage to 70\% (recall the nominal level is 80\%).  

How is this possible? The ensemble is mainly comprised of constituent forecasters
that ignore geographic dependencies between states \cite{cramer2022evaluation}
for the sake of simplicity or computational tractability. But COVID infections
and deaths exhibit strong spatiotemporal dependencies, and most US states
experienced the winter wave of late 2020/early 2021 at similar points in
time. The scorecaster is thus able to learn from the mistakes made on other
US states in order to prospectively adjust the ensemble's forecasts for the
state of California. Similar improvements can be seen for other states, and we
include experiments for New York and Texas as examples in Appendix
\ref{app:covid-more}, which also gives more details on the scorecaster and the
results. 

\paragraph{Electricity demand forecasting.}

Next we consider a data set on electricity demand forecasting in New South
Wales \cite{harries1999splice}, which includes half-hourly data from May 7,
1996 to December 5, 1998. For the base forecaster we use a Transformer model
\cite{vaswani2017attention} as implemented in \texttt{darts}
\cite{herzen2022darts}. This is only re-trained daily, to predict the entire
day's demand in one batch; this is a standard approach with Transformer
models due to their high computational cost. For the conformal score, we use the 
asymmetric (signed) residual score. We use the tan integrator as before, and we
use a lightweight Theta model \cite{assimakopoulos2000theta}, re-trained at
every time point (half-hour), as the scorecaster.     

\begin{figure}[ht]
\includegraphics[width=\textwidth]{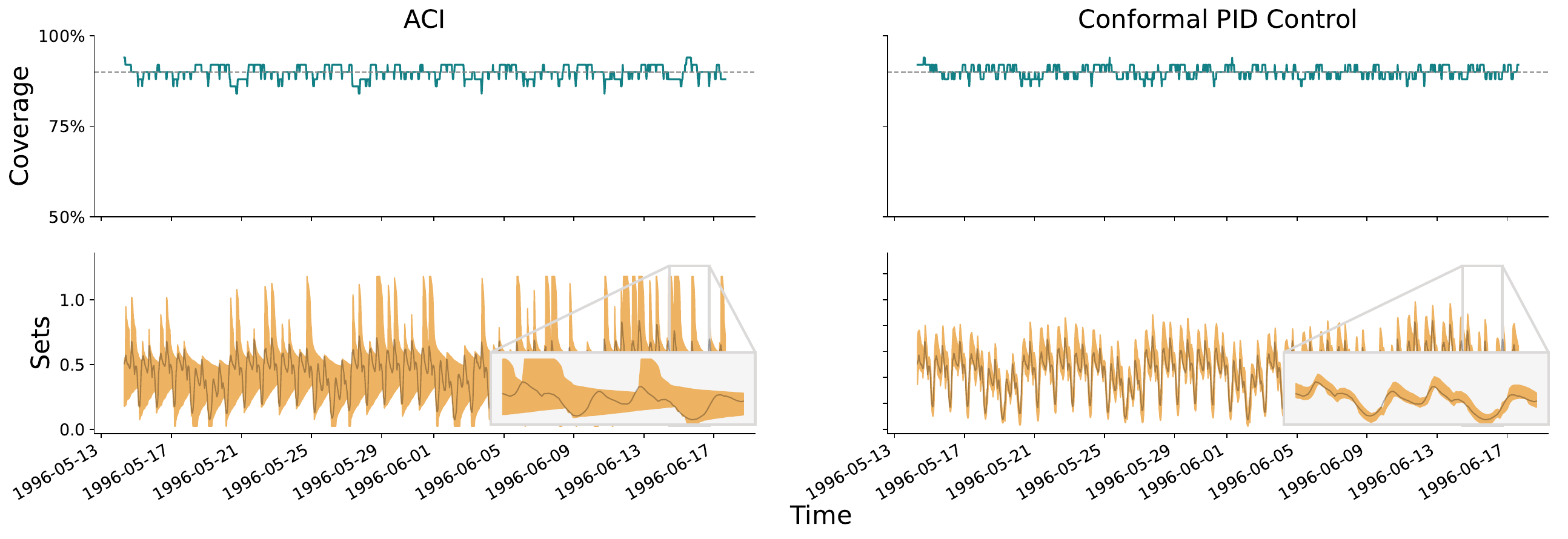}
\caption{Results for electricity demand forecasting. The left column shows
  adaptive conformal inference (ACI), and the right column shows conformal PID control. The base forecaster is a 
  Transformer model, and we use a tan integrator and a Theta scorecaster. The
  format of the figure follows that of Figure \ref{fig:teaser-1}, except the
  nominal coverage is now $1-\alpha = 0.9$, and the coverage is averaged
  over a trailing window of 50 points (we also omit the red dots which mark
  miscoverage events). Summary statistics are available in Table
  \ref{fig:table-teaser-2}.}   
\label{fig:teaser-2}
\end{figure}

The results are shown in the right panel of Figure \ref{fig:teaser-2}, where
adaptive conformal inference (ACI) \cite{gibbs2021adaptive} is also compared in
the left panel. In short, conformal PID control is able to anticipate intraday
variations in the scores, and produces sets that ``hug'' the ground truth
sequence tightly; it achieves tight coverage without generating excessively 
large or infinite sets. The main reason why this is improved is that the
scorecaster has a seasonality component built into its prediction model; in
general, large improvements such as the one exhibited in Figure
\ref{fig:teaser-2} should only be expected when the base forecaster is
imperfect, as is the case here.  

\subsection{Related work}

The adversarial online view of conformal prediction was pioneered by
\cite{gibbs2021adaptive} in the same paper that first introduced ACI.
Since then, there has been significant work towards improving ACI, primarily by
setting the learning rate adaptively \cite{gibbs2022conformal,
  zaffran2022adaptive, bhatnagar2023improved}, and incorporating ideas from
multicalibration to improve conditional coverage \cite{bastani2022practical}.  
It is worth noting that \cite{bhatnagar2023improved} also makes the observation
that the ACI iteration can be generalized to track the quantile of the score
sequence, although their focus is on adaptive regret guarantees. Because the
topic of adaptive learning rates for ACI and related algorithms has already been  
investigated heavily, we do not consider it in the current paper. Any such
method, such as those of \cite{gibbs2022conformal, bhatnagar2023improved} should
work well in conjunction with our proposed algorithms.  

A related but distinct line of work surrounds online \emph{calibration} in the 
adversarial sequence model, which dates back to \cite{foster1999proof,
  foster1998asymptotic}, and connects in interesting ways to both game theory
and online learning. We will not attempt to provide a comprehensive review of
this rich and sizeable literature, but simply highlight
\cite{kuleshov2015calibrated, kuleshov2017estimating, kuleshov2023online} as a 
few interesting examples of recent work. 

Lastly, outside the online setting, we note that several researchers have been
interested in generalizing conformal prediction beyond the i.i.d.\ (or
exchangeable) data setting: this includes \cite{tibshirani2019conformal,
  podkopaev2021distribution, lei2021conformal,  
  fannjiang2022conformal, candes2023conformalized}, and for time series 
prediction, in particular, \cite{chernozhukov2018exact,
  stankeviciute2021conformal, xu2021conformal, xu2023sequential,
  auer2023conformal}. The focus of all of these papers is quite different, and
they all rely on probabilistic assumptions on the data sequence to achieve validity.

\section{Methods}
\label{sec:methods}
We describe the main components of our proposal one at a time, beginning with
the quantile tracker.   

\subsection{Quantile tracking}

The starting point for quantile tracking is to consider the following
optimization problem:
\begin{equation}
\label{eq:quantile-opt}
\minimize_{q \in \R} \; \sum_{t=1}^T \rho_{1-\alpha} (s_t - q), 
\end{equation}
for large $T$, where we abbreviate $s_t = s_t(x_t, y_t)$ for the score of the
test point, and $\rho_{1-\alpha}$ denotes the quantile loss at the level
$1-\alpha$, i.e., $\rho_\tau(z) = \tau |z|$ for $z > 0$ and $(1-\tau) |z|$ for 
$z \leq 0$. The latter is the standard loss used in quantile regression  
\cite{koenker1978regression, koenker2005quantile}. Problem
\eqref{eq:quantile-opt} is thus a simple convex (linear) program that tracks the  
$1-\alpha$ quantile of the score sequence $s_t$, $t \in N$. To see this,
recall that for a continuously distributed random variable $Z$, the expected
loss $\E[\rho_{1-\alpha}(Z - q)]$ is uniquely minimized (over $q \in \R$) at the
level $1-\alpha$ quantile of the distribution of $Z$. 

In the sequential setting, where we receive one score $s_t$ at a time, a natural
and simple approach is to apply \emph{online gradient descent} to
\eqref{eq:quantile-opt}, with a constant learning rate $\eta > 0$. This results
in the update:\footnote{Technically, this is the online subgradient method;
  in a slight abuse of notation, we write $\nabla\rho_{1-\alpha}(0)$ to denote 
  a subgradient of $\rho_{1-\alpha}$ at 0, which can take on any value in
  $[-\alpha, 1-\alpha]$.}   
\begin{align}
\nonumber
q_{t+1} &= q_t + \eta \nabla\rho_{1-\alpha}(s_t - q_t) \\
\label{eq:quantile-tracker}
&= q_t + \eta (\err_t - \alpha),
\end{align}
where the second line follows as $\nabla\rho_{1-\alpha}(s_t - q_t) =  1-\alpha$
if $s_t > q_t \iff \err_t = 1$, and $\nabla\rho_{1-\alpha}(s_t - q_t) = -\alpha$
if $s_t \leq q_t \iff \err_t = 0$. Note that the update in
\eqref{eq:quantile-tracker} is highly intuitive: if we
miscovered (committed an error) at the last iteration then we increase the
quantile, whereas if we covered (did not commit an error) then we decrease 
the quantile.   

Even though it is extremely simple, the quantile tracking iteration
\eqref{eq:quantile-tracker} can achieve long-run coverage own its own, provided
the scores are bounded.   

\begin{proposition}
\label{prop:quantile-tracker}
Let $\{s_t\}_{t \in \N}$ be any sequence of numbers in $[-b,b]$, for $0 < b < 
\infty$. Then the quantile tracking iteration \eqref{eq:quantile-tracker}
satisfies 
\[
\Bigg| \frac{1}{T} \sum_{t=1}^T (\err_t - \alpha) \Bigg| \leq \frac{b +
  \eta}{\eta T}, 
\] 
for any learning rate $\eta > 0$ and $T \geq 1$. In particular, this means
\eqref{eq:quantile-tracker} yields long-run coverage as in
\eqref{eq:long-run-coverage}.   
\end{proposition}

The proof is very simple, and we derive it as a corollary of Proposition
\ref{prop:integrator}, given in the next subsection, because the proof reveals
something perhaps unforeseen about the quantile tracker: it acts as an error
integrator, despite only adjusting the quantile based on the most recent time
step.   

\begin{proof}
Without a loss of generality, set $q_1=0$. Unraveling the iteration
\eqref{eq:quantile-tracker} yields  
\begin{equation}
\label{eq:quantile-tracker-integ}
q_{t+1} = \eta \sum_{i=1}^t (\err_i - \alpha).
\end{equation}
For $r_t(x) = \eta x$, $h(t) = b$, we see \eqref{eq:saturation} holds with
$c = 1/\eta$. Proposition \ref{prop:integrator} now applies.
\end{proof}

A few remarks are in order. First, although Proposition
\ref{prop:quantile-tracker} assumes boundedness of the scores, we do not need to
know this bound in order to run \eqref{eq:quantile-tracker} and obtain long-run
coverage. As long as the scores lie in $[-b,b]$ for any finite $b$, the
guarantee goes through---clearly, the quantile tracker proceeds agnostically and
performs the same updates in any case.  

Second, for the learning rate, in practice we typically set $\eta$
heuristically, as some fraction of the highest score over a trailing window 
\smash{$\hat{B}_t = \max\{ s_{t-\Delta+1}, \dots, s_t \}$}. On this scale,
setting \smash{$\eta = 0.1 \hat{B}_t$} usually gives good results, and we use
it in all experiments unless specified otherwise (we also set the window length
$\Delta$ to be the same as the length of the burn-in period for training the
initial base forecaster and scorecaster).\footnote{Technically, this learning
  rate is not fixed, so Proposition \ref{prop:quantile-tracker} does not
  directly apply. However, we can view it as a special case of error integration
  and an application of Proposition \ref{prop:integrator} thus provides the
  relevant coverage guarantee.}  
Extremely high learning rates result in volatile sets, while very low ones may
fail to keep up with rapid changes in the score distribution.  

Finally, the proof reveals that quantile tracking \eqref{eq:quantile-tracker},
which comes from applying online gradient descent to \eqref{eq:quantile-opt},
can be equivalently viewed as a pure linear integrator
\eqref{eq:quantile-tracker-integ} of past coverage errors.  This explains why
quantile tracking is able to achieve coverage: as we will see later, an error
integrator induces a certain kind of self-correcting behavior: after some amount
of excess cumulative miscoverage it forces a coverage event, and vice versa,
for excess cumulative coverage.

\paragraph{ACI as a special case.}

Though it may not be immediately obvious, adaptive conformal inference (ACI) is
actually a special case of the quantile tracker. ACI follows the iteration: 
\[
\alpha_{t+1} = \alpha_t - \eta (\err_t - \alpha),
\]
which is equivalent to
\begin{align*}
1-\alpha_{t+1} &= 1-\alpha_t + \eta (\err_t - \alpha) \\
&= 1-\alpha_t + \eta \nabla\rho_{1-\alpha}(\beta_t - (1-\alpha_t)), 
\end{align*}
for \smash{$\beta_t = \inf\{ \beta : s_t \leq \quantile{\beta}{\{s_1,\dots,
    s_{t-1}\}}$}. This shows that ACI is a particular instance of the quantile 
tracker that uses a secondary score $s'_t = \beta_t$ and quantile $q'_t = 
1-\alpha_t$. Thus, because quantile tracking \eqref{eq:quantile-tracker} is the
same as a linear coverage integrator \eqref{eq:quantile-tracker-integ}, so 
is ACI.

We can see here that ACI transforms unbounded score sequences into bounded ones,
which then implies long-run coverage for any score sequence. This may, however,
come at a cost: ACI can sometimes output infinite or null prediction sets (when
$\alpha_t$ drifts below 0 or above 1, respectively). Direct quantile tracking
on the scale of the original score sequence does not have this behavior. 

\subsection{Error integration}

Error integration is a generalization of quantile tracking that follows the
iteration: 
\begin{equation}
\label{eq:integrator}
q_{t+1} = r_t \Bigg( \sum_{i=1}^t (\err_i - \alpha) \Bigg), 
\end{equation}
where $r_t$ is a saturation function that satisfies \eqref{eq:saturation} for an
admissible function $h$; recall that we use admissible to mean nonnegative,
nondecreasing, and sublinear. As we saw in \eqref{eq:quantile-tracker-integ},
the quantile tracker uses a \emph{constant} threshold function $h$, whereas $h$
is now permitted to grow, as long as it grows sublinearly, i.e., $h(t)/t \to 0$
as $t \to \infty$. A non-constant threshold function $h$ can be desirable
because it means that $r_t$ will ``saturate'' (will hit the conditions on the
right-hand sides in \eqref{eq:saturation}) less often, so corrections for
coverage error will occur less often, and in this sense, a greater degree of
coverage error can be tolerated along the sequence. 

The next proposition, in particular its proof, makes the role of $h$ precise.  

\begin{proposition}
\label{prop:integrator}
Let $\{s_t\}_{t \in \N}$ be any sequence of numbers in $[-b,b]$, where $b > 0$,
and may be infinite. Assume that $r_t$ satisfies \eqref{eq:saturation}, for an
admissible function $h$. Then the error integration iteration
\eqref{eq:integrator} satisfies  
\begin{equation}
\label{eq:integrator-error}
\Bigg| \frac{1}{T} \sum_{t=1}^T (\err_t - \alpha) \Bigg| \leq \frac{ch(T) +
  1}{T}, 
\end{equation}     
for any $T \geq 1$, where $c$ is the constant in \eqref{eq:saturation}. In
particular, this means \eqref{eq:integrator} yields long-run coverage
\eqref{eq:long-run-coverage}.     
\end{proposition}

\begin{proof}[Proof of Proposition \ref{prop:integrator}]
Abbreviate \smash{$E_T = \sum_{t=1}^T(\err_t - \alpha)$}. We will prove one 
side of the absolute inequality in \eqref{eq:integrator-error}, namely,
\smash{$E_T \leq ch(T) + 1$}, and the other side follows similarly. We use
induction. The base case, for $T=1$, holds trivially. Now assume the result is
true up to $T-1$. We divide the argument into two cases: either $ch(T-1) <
E_{T-1} \leq ch(T-1) + 1$ or \smash{$E_{T-1} \leq ch(T-1)$}. In the first case,
note that that \eqref{eq:saturation} implies \smash{$q_T = r_t(E_{T-1}) \geq b$} 
and therefore $s_T \leq q_T$ and $\err_t = 0$. This means that    
\[
E_T = E_{T-1} - \alpha \leq ch(T-1) + 1 - \alpha \leq ch(T) + 1,
\]
as $h$ is nondecreasing, which is the desired result at $T$. In the second case,
we just use $\err_T \leq 1$, so  
\[
E_T \leq E_{T-1} + 1-\alpha \leq ch(T-1) + 1-\alpha \leq ch(T) + 1.  
\]
This again gives the desired result at $T$, and completes the proof.
\end{proof}

Importantly, Proposition \ref{prop:integrator} suffices to prove Theorem
\ref{thm:main}. 

\begin{proof}[Proof of Theorem \ref{thm:main}]
We can transform \eqref{eq:main} by setting
\smash{$q'_{t+1} = q_{t+1} - \hat{q}_{t+1}$}, and this becomes an update of the
form \eqref{eq:integrator} with respect to $q'_{t+1}$. Further, the score
sequence in this new parameterization is \smash{$s'_t = s_t - \hat{q}_t$}, which
is in $[-b,b]$ because both $s_t$ and \smash{$\hat{q}_t$} are in $[-b/2,b/2]$.  
Applying Proposition \ref{prop:integrator} gives the result.
\end{proof}

As already mentioned in the introduction, in all our experiments we use a
nonlinear saturation function \smash{$r_t(x) = K_{\text{I}} \tan(x
  \log(t) /(t C_{\text{sat}}))$}, where we set $\tan(x) = \sign(x) \cdot
\infty$ for $x \notin [-\pi/2, \pi/2]$, and \smash{$C_{\text{sat}},
  K_{\text{I}}>0$} are constants that we choose heuristically (described in
Appendix \ref{app:heuristics}). In a sense, this tan integrator is akin to a 
quantile tracker whose learning rate adapts to the current coverage gap. To see 
this, we can use a first-order Taylor approximation, which shows (ignoring
constants):  
\[
q_{t+1} = \tan\Bigg( \frac{\log(t)}{t} \sum_{i=1}^t (\err_i - \alpha) \Bigg)   
\approx q_t + \underbrace{\frac{\log(t)}{t} \sec^2 \Bigg( \frac{\log(t-1)}{t-1}  
  \sum_{i=1}^{t-1} (\err_i - \alpha) \Bigg)}_{\text{effective learning rate}}
\nabla \rho_{1-\alpha}(s_t - q_t) .  
\]
Above, $\sec(x) = 1/\cos(x)$ is the secant function, which has a U-shape; thus
we can see from the above that the effective learning rate is higher for larger
errors. Similar analyses for different integrators will give different adaptive
learning rates; see Appendix \ref{app:learning-rates} for another example.   

\subsection{Scorecasting}

The final piece to discuss is scorecasting. A scorecaster attempts to
forecast $q_{t+1}$ directly, taking advantage of any leftover signal that is not 
captured by the base forecaster.
This is the role played by
\smash{$\hat{q}_{t+1}$} in \eqref{eq:main}.    
Scorecasting may be particularly useful when it is difficult to modify or
re-train the base forecaster. This can occur when the base forecaster is
computationally costly to train (e.g., as in a Transformer model); or it can
occur in complex operational prediction pipelines where frequently updating a
forecasting implementation is infeasible. Another scenario where scorecasting
may be useful is one in which the forecaster and scorecaster have access to
different levels of data. For example, if a public health agency collects epidemic
forecasts from external groups, and forms an ensemble forecast, then the agency 
may have access to finer-grained data that it can use to recalibrate the
ensemble's prediction sets (compared to the level of data granularity granted to
the forecasters originally).  

This motivates the need for scorecasting as a modular layer that ``sits on top''
of the base forecaster and residualizes out systematic errors in the score 
distribution. This intuition is made more precise by recalling, as described
above (following Proposition \ref{prop:integrator}), that scorecasting combined
with error integration as in~\eqref{eq:main} is just a
reparameterization of error integration \eqref{eq:integrator}, where
\smash{$q'_t = q_t - \hat{q}_t$} and \smash{$s'_t = s_t - \hat{q}_t$} are the
new quantile and new score, respectively. A well-executed scorecaster could
reduce the variability in the scores and make them more exchangeable, resulting
in more stable coverage and tighter prediction sets, as seen in Figure
\ref{fig:teaser-2}. On the other hand, using an aggressive scorecaster in a
situation in which there is little or no signal left in the scores can actually
hurt: in this case it would only add variance to the new score sequence $s'_t$,
which could result in more volatile coverage and larger sets.   

There is no limit to what we can choose for the scorecasting model. We might
like to use a model that can simultaneously incorporate seasonality, trends, and 
exogenous covariates. Two traditional choices would be SARIMA (seasonal
autoregressive integrated moving average) and ETS (error-trend-seasonality)
models, but there are many other available methods, such as the Theta model
\cite{assimakopoulos2000theta}, Prophet model \cite{taylor2018forecasting}, and 
various neural network forecasters. A modern review of forecasting methods is
given in \cite{hyndman2018forecasting}.     

\subsection{Putting it all together}

Briefly, we revisit the PID perspective, to recap how quantile tracking, error
integration, and scorecasting fit in and work in combination. It helps to
return to \eqref{eq:pid-update}, which we copy again here for convenience:  
\begin{equation}
\label{eq:pid-revisit}
q_{t+1} = g'_t + \eta (\err_t - \alpha) + r_t \Bigg( \sum_{i=1}^t (\err_i -
\alpha) \Bigg).  
\end{equation}
Quantile tracking is precisely given by taking $g'_t = q_t$ and $r_t = 0$. This
can be seen as equivalent to P control: subtract $q_t$ from both sides in
\eqref{eq:pid-revisit} and treat the increment $u_{t+1} = q_{t+1} - q_t$ as the 
process variable; then in this modified system, quantile tracking is exactly P
control. For this reason, we use ``conformal P control'' to refer to the
quantile tracker in the experiments that follow. Similarly, we use ``conformal
PI control'' to refer to the choice $g'_t = q_t$, and $r_t \not = 0$ as a 
generic integrator (for us, tan is the default). Lastly, ``conformal PID
control'' refers to letting $g'_t$ be a generic scorecaster, and $r_t \not= 0$
be a generic integrator.

\section{Experiments}
\label{sec:experiments}
In addition to the statewide COVID-19 death forecasting experiment described in
the introduction, we run experiments on all combinations of the following data
sets and forecasters.  

\begin{multicols}{2}
Data sets:
\begin{itemize}
\item Electricity demand in New South Wales \cite{harries1999splice} 
\item Return (log price) of Amazon, Google, and Microsoft stock
  \cite{nguyen2018stock} 
\item Temperature in Delhi \cite{vrao2017climate} 
\end{itemize}

\columnbreak

Forecasters (all via \texttt{darts} \cite{herzen2022darts}) :
\begin{itemize}
\item Autoregressive (AR) model with 3 lags
\item Theta model with $\theta=2$ \cite{assimakopoulos2000theta} 
\item Prophet model \cite{taylor2018forecasting}
\item Transformer model \cite{vaswani2017attention}
\end{itemize}
\end{multicols}

In all cases except for the COVID-19 forecasting data set, we: re-train the
base forecaster at each time point; construct prediction sets using the
asymmetric (signed) residual score; and use a Theta model for the
scorecaster. For the COVID-19 forecasting setting, we: use the given ensemble
model as the base forecaster (no training at all); construct prediction sets
using the asymmetric quantile score; and use an $\ell_1$-penalized quantile  
regression as the scorecaster, fit on features derived from previous forecasts,
cases, and deaths, as described in the introduction. And lastly, in all cases,
we use a tan function for the integrator with constants chosen heuristically, as
described in Appendx \ref{app:heuristics}.

The results that we choose to show in the subsections below are meant to
illustrate key conceptual points (differences between the methods). Additional
results are presented in Appendix \ref{app:further-experiments}. Our GitHub
repository, \url{https://github.com/aangelopoulos/conformal-time-series},
provides the full suite of evaluations.  

\subsection{ACI versus quantile tracking} 
\label{subsec:aci-vs-qt}

\begin{figure}[p]
\centering
\includegraphics[width=\linewidth]{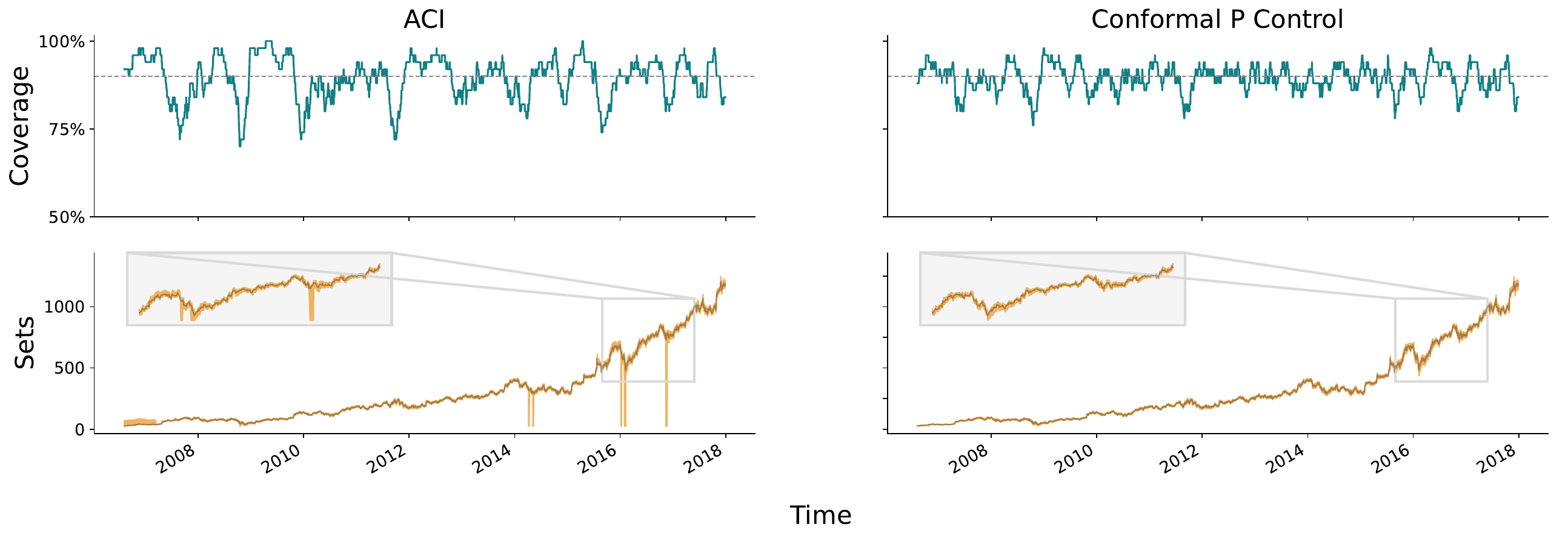}

\smallskip\smallskip
{\footnotesize \begin{tabular}{lllllllll}
\toprule
& \multicolumn{2}{c}{AR}& \multicolumn{2}{c}{Prophet}& \multicolumn{2}{c}{Theta}& \multicolumn{2}{c}{Transformer} \\
& ACI & P Ctrl & ACI & P Ctrl & ACI & P Ctrl & ACI & P Ctrl \\
\midrule
Marginal coverage & 0.894 & 0.894 & 0.904 & 0.888 & 0.895 & 0.894 & 0.906 & 0.887 \\
Longest err sequence & 6 & 3 & 13 & 6 & 5 & 5 & 21 & 9 \\
Average set size & $\infty$ & 17.6 & $\infty$ & 51.7 & $\infty$ & 17.8 & $\infty$ & 70.4 \\
Median set size & 13.4 & 13.3 & 50.7 & 37.3 & 12.8 & 13.1 & 61.9 & 44.3 \\
75\% quantile set size & 28.4 & 22.3 & 117 & 72.1 & 20.9 & 22.6 & 179 & 98.4 \\
90\% quantile set size & 44.1 & 37.7 & 236 & 114 & 38.6 & 38.4 & 302 & 153 \\
95\% quantile set size & 49.9 & 46.2 & $\infty$ & 140 & 48.2 & 46.9 & $\infty$ & 196 \\
\bottomrule
\end{tabular}
} 
\caption{Results for forecasting Amazon stock return, comparing ACI and quantile
  tracking (P control). The plots show AR as the base forecaster; the table
  summarizes the results of all four base forecasters. We use the default
  learning rates for both ACI and quantile tracking: $\eta = 0.005$ and
  \smash{$\eta = 0.1 \hat{B}_t$}, respectively.} 
\label{fig:aci-vs-qt-1}

\bigskip\bigskip\bigskip
\includegraphics[width=\linewidth]{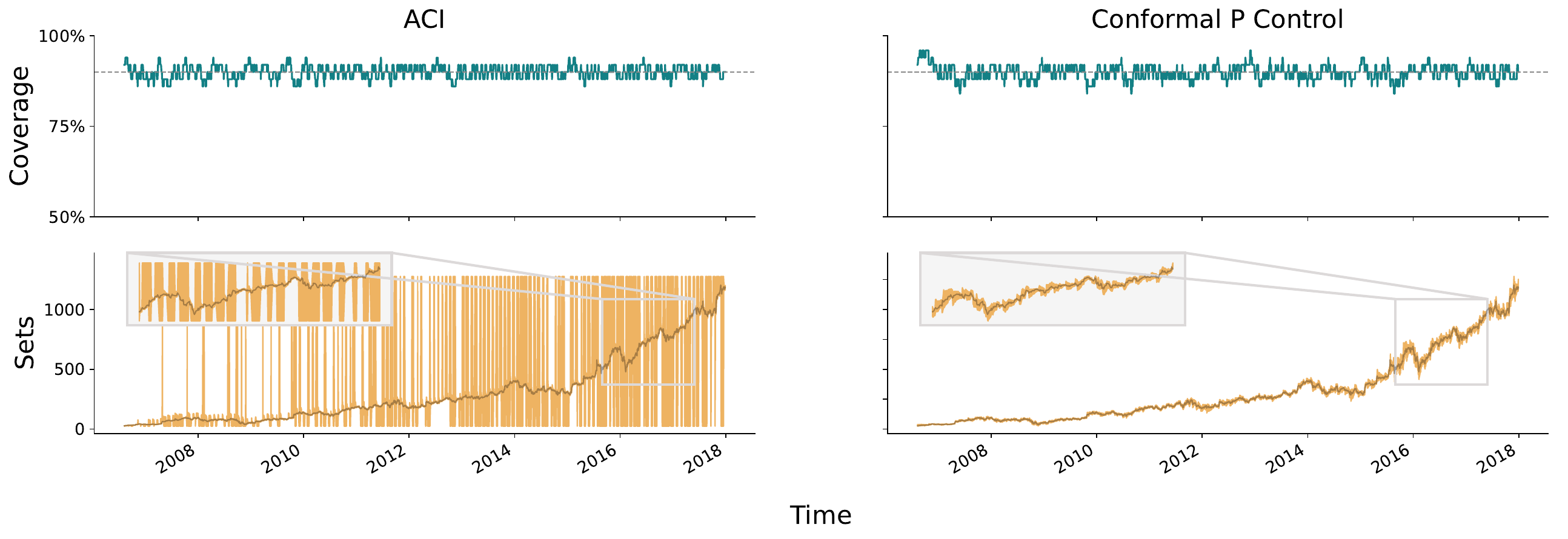}

\smallskip\smallskip
{\footnotesize \begin{tabular}{lllllllll}
\toprule
& \multicolumn{2}{c}{AR}& \multicolumn{2}{c}{Prophet}& \multicolumn{2}{c}{Theta}& \multicolumn{2}{c}{Transformer} \\
& ACI & P Ctrl & ACI & P Ctrl & ACI & P Ctrl & ACI & P Ctrl \\
\midrule
Marginal coverage & 0.9 & 0.898 & 0.901 & 0.897 & 0.9 & 0.896 & 0.902 & 0.897 \\
Longest err sequence & 2 & 2 & 8 & 3 & 2 & 2 & 7 & 3 \\
Average set size & $\infty$ & 28.4 & $\infty$ & 42.6 & $\infty$ & 27.8 & $\infty$ & 54.2 \\
Median set size & $\infty$ & 18.6 & 60.6 & 31.9 & $\infty$ & 18.7 & 46 & 35.5 \\
75\% quantile set size & $\infty$ & 37.4 & $\infty$ & 56.5 & $\infty$ & 37.8 & $\infty$ & 69.7 \\
90\% quantile set size & $\infty$ & 66.3 & $\infty$ & 93.4 & $\infty$ & 63.4 & $\infty$ & 123 \\
95\% quantile set size & $\infty$ & 81.8 & $\infty$ & 116 & $\infty$ & 78.5 & $\infty$ & 164 \\
\bottomrule
\end{tabular}
}
\caption{As in Figure \ref{fig:aci-vs-qt-1}, but with larger learning rates for 
  ACI and quantile tracking: $\eta = 0.1$ and \smash{$\eta = 0.5 \hat{B}_t$},
  respectively.}
\label{fig:aci-vs-qt-2}
\end{figure}

We forecast the daily Amazon (AMZN) opening stock price from 2006--2014. We do
this in log-space (hence predicting the return of the stock). Figure
\ref{fig:aci-vs-qt-1} compares ACI and the quantile tracker, each with its
default learning rate: $\eta = 0.005$ for ACI, and \smash{$\eta = 0.1
  \hat{B}_t$} for quantile tracking. We see that the coverage from each method
is decent, but oscillates nontrivially around the nominal level of $1-\alpha =
0.9$ (with ACI generally having larger oscillations). Figure
\ref{fig:aci-vs-qt-2} thus increases the learning rate for each method: $\eta =
0.1$ for ACI, and \smash{$\eta = 0.5 \hat{B}_t$} for the quantile tracker. We
now see that both deliver very tight coverage. However, ACI does so by
frequently returning infinite sets; meanwhile, the corrections to the sets made
by the quantile tracker are nowhere near as aggressive. 

As a final comparison, in Appendix \ref{app:clipped-aci-vs-qt}, we modify ACI 
to clip the sets in a way that disallows them from ever being infinite. This
heuristic may be used by practitioners that want to guard against infinite sets,
but it no longer has a validity guarantee for bounded or unbounded scores. The
results in the appendix indicate that the quantile tracker has similar coverage
to this procedure, and usually with smaller sets.

\subsection{The effect of integration}

Next we forecast the daily Google (GOOGL) opening stock price from 2006--2014
(again done in log-space). Figure \ref{fig:qt-vs-integrator} compares the
quantile tracker without and without an additional integrator component (P
control versus PI control). We purposely choose a very small learning rate,
\smash{$\eta = 0.01 \hat{B}_t$}, in order to show how the integrator can
stabilize coverage, which it does nicely for most of the time series. The
coverage of PI control begins to oscillate more towards the end of the sequence,
which we attribute at least in part to the fact that the integrator measures
coverage errors accumulated over \emph{all time}---and by the end of a long
sequence, the marginal coverage can still be close to $1-\alpha$ even if the
local coverage deviates more wildly. This can be addressed by using a local
version of the integrator, an idea we return to in the discussion. 

\begin{figure}[ht]
\centering
\includegraphics[width=\linewidth]{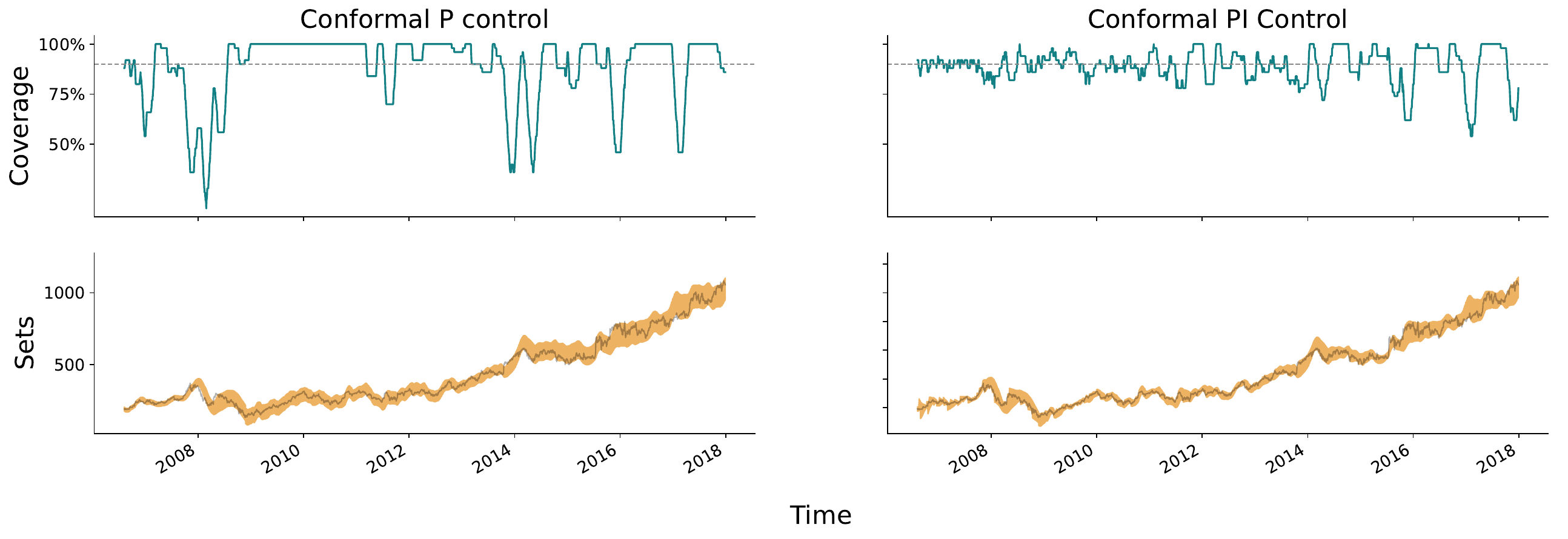}

\smallskip\smallskip
{\footnotesize \begin{tabular}{lllllllll}
\toprule
& \multicolumn{2}{c}{AR}& \multicolumn{2}{c}{Prophet}& \multicolumn{2}{c}{Theta}& \multicolumn{2}{c}{Transformer} \\
& P Ctrl & PI Ctrl & P Ctrl & PI Ctrl & P Ctrl & PI Ctrl & P Ctrl & PI Ctrl \\
\midrule
Marginal coverage & 0.895 & 0.896 & 0.882 & 0.888 & 0.891 & 0.896 & 0.817 & 0.897 \\
Longest err sequence & 6 & 3 & 31 & 13 & 58 & 6 & 106 & 10 \\
Average set size & 19.2 & 24.7 & 95.6 & 68.2 & 42.9 & 33.5 & 281 & 108 \\
Median set size & 16.4 & 22.8 & 91.5 & 60.7 & 44.3 & 32.8 & 219 & 73.8 \\
75\% quantile set size & 21.3 & 30.1 & 120 & 88 & 60 & 42.6 & 386 & 121 \\
90\% quantile set size & 30.7 & 38.4 & 150 & 118 & 66.2 & 51.3 & 482 & 244 \\
95\% quantile set size & 32.8 & 44.7 & 156 & 139 & 69.1 & 57.5 & 524 & 341 \\
\bottomrule
\end{tabular}
} 
\caption{Results for forecasting Google stock return, comparing quantile
  tracking with and without the integrator (P control versus PI control). The
  plots show Prophet as the base forecaster; the table summarizes the results
  of all four base forecasters. We purposely use a very small learning rate,
  \smash{$\eta = 0.01 \hat{B}_t$}, in order to show how the integrator can 
  stabilize coverage.}   
\label{fig:qt-vs-integrator}
\end{figure}

\subsection{The effect of scorecasting}

Figures \ref{fig:teaser-1} and \ref{fig:teaser-2} already showcase examples in
which scorecasting offers significant improvement in coverage and set
sizes. Recall that these were settings in which the base forecaster produces
errors (scores) that have predictable trends. Further examples in the COVID-19 
forecasting setting, which display similar benefits to scorecasting, are given
in Appendix \ref{app:covid-more}.     

We emphasize that it is not always the case that scorecasting will help. In some
settings, scorecasting may introduce enough variance into the new score sequence
that the coverage or sets will degrade in stability. (For example, this will
happen if we run a highly complex scorecaster on a sequence of i.i.d.\ scores,
where there are no trends whatsoever.) In practice, scorecasters should be
designed with care, just as one would design a base forecaster; it is unlikely
that using ``out of the box'' techniques for scorecasting will be robust enough,
especially in high-stakes problems. Appendix \ref{app:further-experiments}
provides examples in which scorecasting, run across all settings using a generic
Theta model, can hurt (for example, it adds noticeable variance to the coverage
and sets in some instances within the Amazon data setting).

\section{Discussion}
Our work presents a framework for constructing prediction sets in time series
that is analogous (and indeed formally equivalent) to PID control. The framework
consists of quantile tracking (P control), which is simply online gradient
descent applied to the quantile loss; error integration (I control) to stabilize
coverage; and scorecasting (D control) to remove systematic trends in the scores
(errors made by the base forecaster).  

We found the combination of quantile tracking and integration to consistently
yield robust and favorable performance in our experiments. Scorecasting provides
additional benefits if there are trends left in scores that are predictable (and
the scorecaster is well-designed), as is the case in some of our
examples. Otherwise, scorecasting may add variability and make the coverage and
prediction sets more volatile. Overall, designing the scorecaster (which
includes the choice to even use one at all) is an important modeling step,
just like the design of the base forecaster.

It is worth emphasizing that, with the exception of the COVID-19 forecasting
example, our experiments are intended to be illustrative and we did not look
to use state-of-the-art forecasters, or include any and all possibly relevant
features for prediction. Further, while we found that using heuristics to set 
constants (such as the learning rate $\eta$, and constants
\smash{$C_{\text{sat}}, K_{\text{I}}$} for the tan integrator) worked decently 
well, we believe that more rigorous techniques, along the lines of
\cite{gibbs2022conformal, bhatnagar2023improved}, can be used to tune these 
adaptively in an online fashion.

We now present an extension of our analysis to conformal risk
control \cite{angelopoulos2022conformal, bates2021distribution,
  feldman2022achieving}. In this problem setting, we are given a sequence of
loss functions \smash{$L_t : 2^{\cY} \times \cY \to [0,1]$} satisfying
$L_t(\cY, y) = 0$ for all $y$, and $L_t(\emptyset, y) = 1$ for all $y$. The goal
is to bound the deviation of the average risk \smash{$\frac{1}{T} \sum_{t=1}^T 
  L_t(\cC_t, y_t)$} from $\alpha$. We state a result for the integrator below,
and give its proof in Appendix \ref{app:risk-control}.  

\begin{proposition}
\label{prop:integrator-rc}
Consider the iteration \smash{$q_{t+1} = r_t (\sum_{i=1}^t (L_i(\cC_i, y_i) -
  \alpha))$}, with $L_t$ as above. Assume that $r_t$ satisfies
\eqref{eq:saturation}, for an admissible function $h$. Also assume that 
\smash{$C_t(\cC_t, y_t) = \emptyset$} if $q_t \leq -b$ and $\cY$ if $q_t \geq
b$, where $b > 0$, and may be infinite. Then for all $T \geq 1$, 
\begin{equation}
\label{eq:integrator-error-rc}
\Bigg|\frac{1}{T} \sum_{t=1}^T (L_t(\cC_t, y_t) - \alpha) \Bigg| \leq
\frac{ch(T) + 1}{T}. 
\end{equation}
for any $T \geq 1$, where $c$ is the constant in \eqref{eq:saturation}. 
\end{proposition}

We briefly conclude by mentioning that we believe many other extensions are
possible, especially with respect to the integrator. Broadly, we can choose to 
integrate in a kernel-weighted fashion
\begin{equation}
\label{eq:kernel-integrator}
r_t \Bigg( \sum_{i=1}^t (\err_i - \alpha) \cdot K \big( (i, x_i, y_i), (t, x_t,
y_t) \big) \Bigg)
\end{equation}
As a special case, the kernel could simply assign weight 1 if $t-i \leq w$, and
weight 0 otherwise, which would result in an integrator that aggregates coverage
over a trailing window of length $w$. This can help consistently sustain better
local coverage, for long sequences. As another special case, the kernel could  
assign a weight based on whether $x_i$ and $x_t$ lie in the same bin in some
pre-defined binning of $\cX$ space, which may be useful for problems with group
structure (where we want group-wise coverage). Many other choices and forms of
kernels are possible, and it would be interesting to consider adding together a
number of such choices \eqref{eq:kernel-integrator} in combination, in a
multi-resolution flavor, for the ultimate quantile update.

\bibliographystyle{alpha}
\bibliography{bibliography}

\clearpage
\appendix

\section{Conformal risk control guarantee}
\label{app:risk-control}

\begin{proof}[Proof of Proposition \ref{prop:integrator-rc}]
The proof is similar to that of Proposition \ref{prop:integrator}---as in  
that proof, we only prove one side of the absolute inequality
\eqref{eq:integrator-error-rc}, and use induction. Abbreviate \smash{$E_T = 
  \sum_{t=1}^T (L_i(\cC_i, y_i) - \alpha))$}. The base case holds trivially. For
the inductive step, either $ch(T-1) < E_{T-1} \leq ch(T-1)+1$ or $E_{T-1} \leq 
ch(T-1)$. In the first case, we have saturated, so $L_t(\cC_t, y_t) = 0$, and   
\[
E_T = E_{T-1} - \alpha \leq ch(T-1) + 1 - \alpha \leq ch(T) + 1,
\]
as $h$ is nondecreasing, which is the desired result at $T$. In the second case,
we just use the boundedness of the loss $L_t(\cC_t,y_t) \leq 1$, so   
\[
E_T \leq E_{T-1} + 1-\alpha \leq ch(T-1) + 1-\alpha \leq ch(T) + 1.  
\]
This again gives the desired result at $T$, and completes the proof.
\end{proof}

\section{Heuristics for setting constants}
\label{app:heuristics}

Consider the tan integrator \smash{$r_t(x) = K_{\text{I}} \tan(x \log(t) /(t
  C_{\text{sat}}))$}, where we set $\tan(x) = \sign(x) \cdot \infty$ whenever $x
\notin [-\pi/2, \pi/2]$, and \smash{$C_{\text{sat}}, K_{\text{I}}>0$} are
constants. The constant \smash{$C_{\text{sat}}$} is primarily in charge of
guaranteeing that by time $T$, we want to have an absolute guarantee of at least
$1 - \alpha - \delta$ coverage. Then we can set 
\[ 
C_{\text{sat}} = \frac{2}{\pi}\big( \lceil\log(T)\delta\rceil - 1/\log(T) \big)
\]
to ensure the tan function has an asymptote at the correct point. 
The purpose of the constant \smash{$K_{\text{I}}$} is to place the integrator 
on the same scale as the scores. So if $B'$ is a hypothesized bound on the 
magnitude of the scores, then one can set \smash{$K_{\text{I}} = B'$}. In
practice, these heuristics can be taken as a starting place, and then the
numbers can be fine-tuned during a burn-in period by hand or algorithmically. As alluded to previously, we
believe there is room for work in the style of \cite{gibbs2022conformal,
  bhatnagar2023improved} to rigorously tune these parameters online, but it is
not the focus of our paper.   

\section{Quantile tracking with decaying learning rate} 
\label{app:learning-rates}

Consider \smash{$r_t(x) = \eta x/\sqrt{t}$}. (This will give long-run coverage 
only for bounded scores, because condition \eqref{eq:saturation} is only met for
finite and not infinite $b$.) Then \eqref{eq:integrator} becomes \smash{$q_{t+1}
  = \frac{\eta}{\sqrt{t}} \sum_{i=1}^t (\err_i - \alpha)$}, which can be
rewritten as 
\[
q_{t+1} = \frac{\sqrt{t-1}}{\sqrt{t}} \frac{\eta}{\sqrt{t-1}} \sum_{i=1}^{t-1}
(\err_i - \alpha) + \frac{\eta}{\sqrt{t}}(\err_i - \alpha) \approx q_t +
\frac{\eta}{\sqrt{t}}(\err_{t} - \alpha). 
\]
This is approximately the quantile tracker \eqref{eq:quantile-tracker} with a
decaying learning rate, on the order of \smash{$1/\sqrt{t}$}.  

\section{Comparison to clipped ACI} 
\label{app:clipped-aci-vs-qt}

Figures \ref{fig:clipped-1} and \ref{fig:clipped-2} compare the quantile tracker
to a clipped version of ACI which disallows infinite-sized sets by clipping the
sets to the largest score seen so far. 

\begin{figure}[p]
\centering
\includegraphics[width=\linewidth]{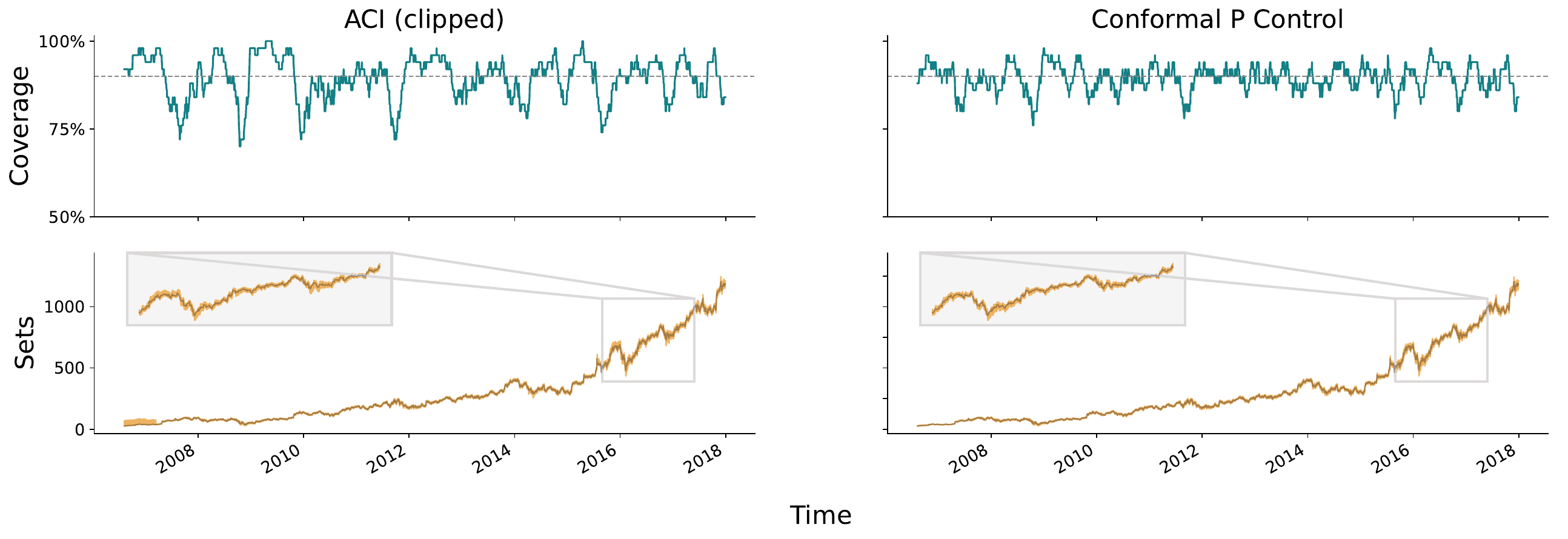}

\smallskip\smallskip
{\footnotesize \begin{tabular}{lllllllll}
\toprule
& \multicolumn{2}{c}{AR}& \multicolumn{2}{c}{Prophet}& \multicolumn{2}{c}{Theta}& \multicolumn{2}{c}{Transformer} \\
& ACI (clipped) & P Ctrl & ACI (clipped) & P Ctrl & ACI (clipped) & P Ctrl & ACI (clipped) & P Ctrl \\
\midrule
Marginal coverage & 0.898 & 0.898 & 0.884 & 0.897 & 0.898 & 0.896 & 0.884 & 0.897 \\
Longest err sequence & 2 & 2 & 8 & 3 & 2 & 2 & 7 & 3 \\
Average set size & 44.9 & 28.4 & 52.6 & 42.6 & 43.3 & 27.8 & 60.5 & 54.2 \\
Median set size & 41.8 & 18.6 & 38.8 & 31.9 & 27.3 & 18.7 & 36.6 & 35.5 \\
75\% quantile set size & 58.9 & 37.4 & 66.9 & 56.5 & 59.5 & 37.8 & 85.5 & 69.7 \\
90\% quantile set size & 93.9 & 66.3 & 137 & 93.4 & 94.7 & 63.4 & 148 & 123 \\
95\% quantile set size & 136 & 81.8 & 166 & 116 & 136 & 78.5 & 182 & 164 \\
\bottomrule
\end{tabular}
}
\caption{As in Figure \ref{fig:aci-vs-qt-1}, but with clipped ACI.}
\label{fig:clipped-1}

\bigskip\bigskip\bigskip
\includegraphics[width=\linewidth]{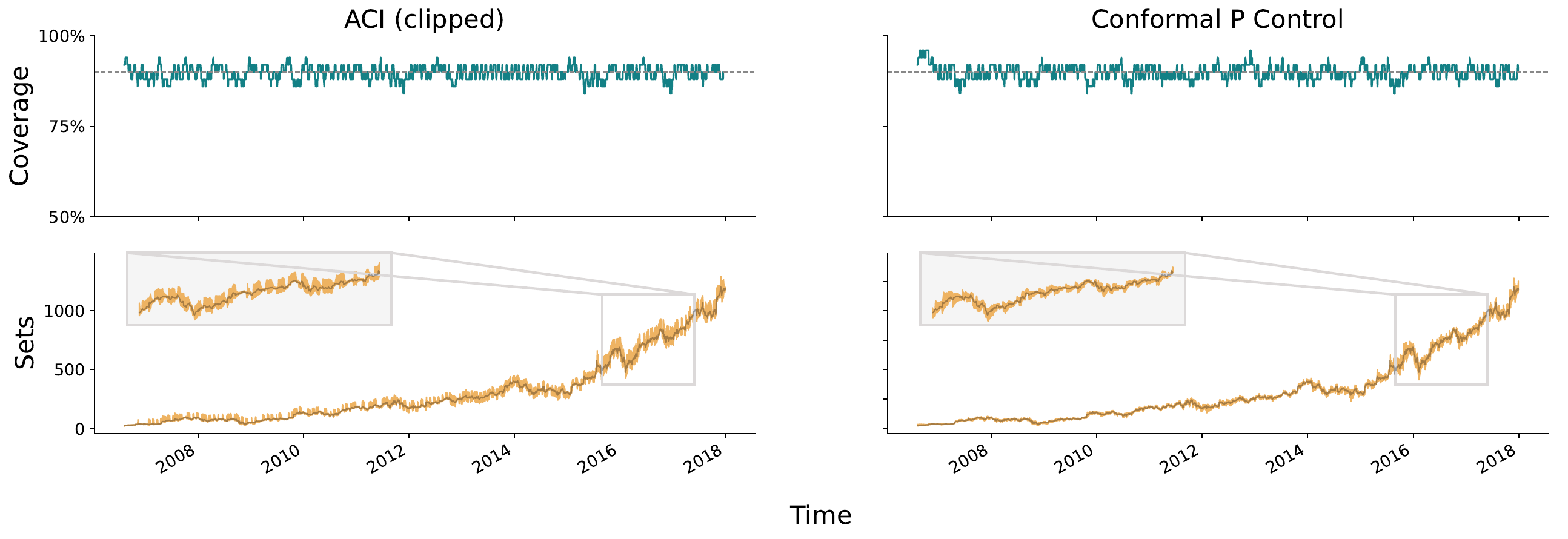}

\smallskip\smallskip
{\footnotesize  \begin{tabular}{lllllllll}
\toprule
& \multicolumn{2}{c}{AR}& \multicolumn{2}{c}{Prophet}& \multicolumn{2}{c}{Theta}& \multicolumn{2}{c}{Transformer} \\
& ACI (clipped) & P Ctrl & ACI (clipped) & P Ctrl & ACI (clipped) & P Ctrl & ACI (clipped) & P Ctrl \\
\midrule
Marginal coverage & 0.894 & 0.894 & 0.896 & 0.888 & 0.895 & 0.894 & 0.89 & 0.887 \\
Longest err sequence & 6 & 3 & 13 & 6 & 5 & 5 & 21 & 9 \\
Average set size & 19.5 & 17.6 & 69.5 & 51.7 & 17.9 & 17.8 & 115 & 70.4 \\
Median set size & 13.4 & 13.3 & 48.7 & 37.3 & 12.8 & 13.1 & 61.7 & 44.3 \\
75\% quantile set size & 27.9 & 22.3 & 91.1 & 72.1 & 20.7 & 22.6 & 165 & 98.4 \\
90\% quantile set size & 44 & 37.7 & 168 & 114 & 38.5 & 38.4 & 248 & 153 \\
95\% quantile set size & 48.7 & 46.2 & 195 & 140 & 47.2 & 46.9 & 304 & 196 \\
\bottomrule
\end{tabular}
}
\caption{As in Figure \ref{fig:aci-vs-qt-2}, but with clipped ACI.}
\label{fig:clipped-2}
\end{figure}

\section{More details on COVID-19 forecasting}
\label{app:covid-more}

In this experiment, the scorecaster receives as input the three most recent
scores (i.e., quantile errors) of the ensemble forecaster, as well as the three 
most recent case and death counts, from \emph{all 50 states}. The scorecaster is
an $\ell_1$-penalized quantile regression as implemented by
\texttt{sklearn.linear\_model.QuantileRegressor}. We fixed tuning parameter for
the  $\ell_1$ penalty at 10; in our experience, the performance of the
scorecaster was fairly robust to this choice. Automatic selection (e.g., using
cross-validation) could be the topic of future study. Figures \ref{fig:ny} and
\ref{fig:tx} shows the analogous experiments but for forecasting in New York and
Texas. 

\begin{table}[t]
\centering
{\footnotesize \begin{tabular}{lll}
\toprule
& Base Forecaster & Conformal PID Ctrl \\
\midrule
Marginal coverage & 0.82 & 0.86 \\
Longest err sequence & 6 & 2 \\
Average set size & 625 & 858 \\
Median set size & 512 & 688 \\
75\% quantile set size & 754 & 1.01e+03 \\
90\% quantile set size & 1.12e+03 & 1.47e+03 \\
95\% quantile set size & 1.45e+03 & 1.8e+03 \\
\bottomrule
\end{tabular}
}
\caption{Summary statistics for COVID-19 death forecasting in California, as in
  Figure \ref{fig:teaser-1}.}
\label{fig:table-teaser-1}

\bigskip\bigskip
{\footnotesize \begin{tabular}{lllll}
\toprule
& \multicolumn{2}{c}{AR}& \multicolumn{2}{c}{Transformer} \\
& ACI & Conformal PID Control & ACI & Conformal PID Ctrl \\
\midrule
Marginal coverage & 0.899 & 0.9 & 0.899 & 0.901 \\
Longest err sequence & 3 & 2 & 3 & 2 \\
Average set size & $\infty$ & 0.177 & $\infty$ & 0.174 \\
Median set size & 0.406 & 0.178 & 0.426 & 0.175 \\
75\% quantile set size & 0.484 & 0.21 & 0.574 & 0.206 \\
90\% quantile set size & 0.672 & 0.236 & $\infty$ & 0.233 \\
95\% quantile set size & $\infty$ & 0.252 & $\infty$ & 0.249 \\
\bottomrule
\end{tabular}
}
\caption{Summary statistics for electricity forecasting, as in Figure
  \ref{fig:teaser-2}. Results for the Prophet and Theta models are not available
  because \texttt{darts} does not support intermittent retraining for these
  algorithms.} 
\label{fig:table-teaser-2}
\end{table}

\begin{figure}[p]
\centering
\includegraphics[width=\textwidth]{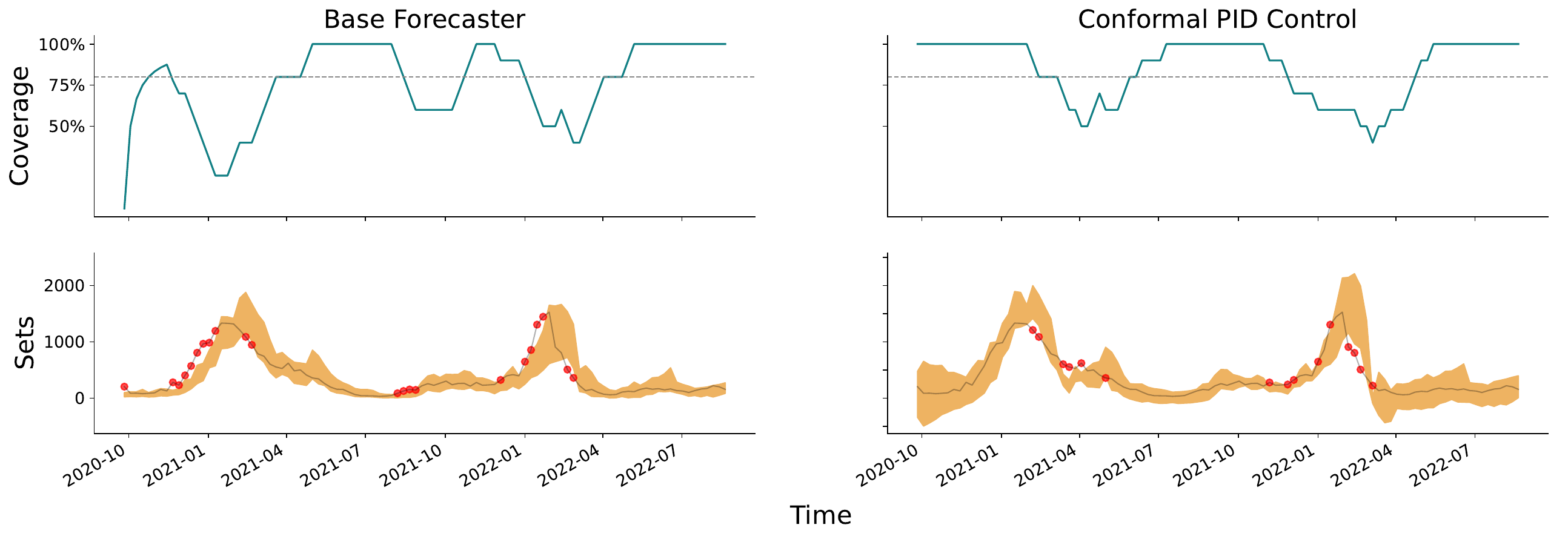}

\smallskip\smallskip
{\footnotesize \begin{tabular}{lll}
\toprule
& Base Forecaster & Conformal PID Ctrl \\
\midrule
Marginal coverage & 0.78 & 0.85 \\
Longest err sequence & 8 & 3 \\
Average set size & 310 & 472 \\
Median set size & 249 & 440 \\
75\% quantile set size & 388 & 608 \\
90\% quantile set size & 667 & 821 \\
95\% quantile set size & 754 & 1.03e+03 \\
\bottomrule
\end{tabular}
}
\caption{Results for 4-week ahead COVID-19 death forecasting in New York.}
\label{fig:ny}

\bigskip\bigskip\bigskip
\includegraphics[width=\textwidth]{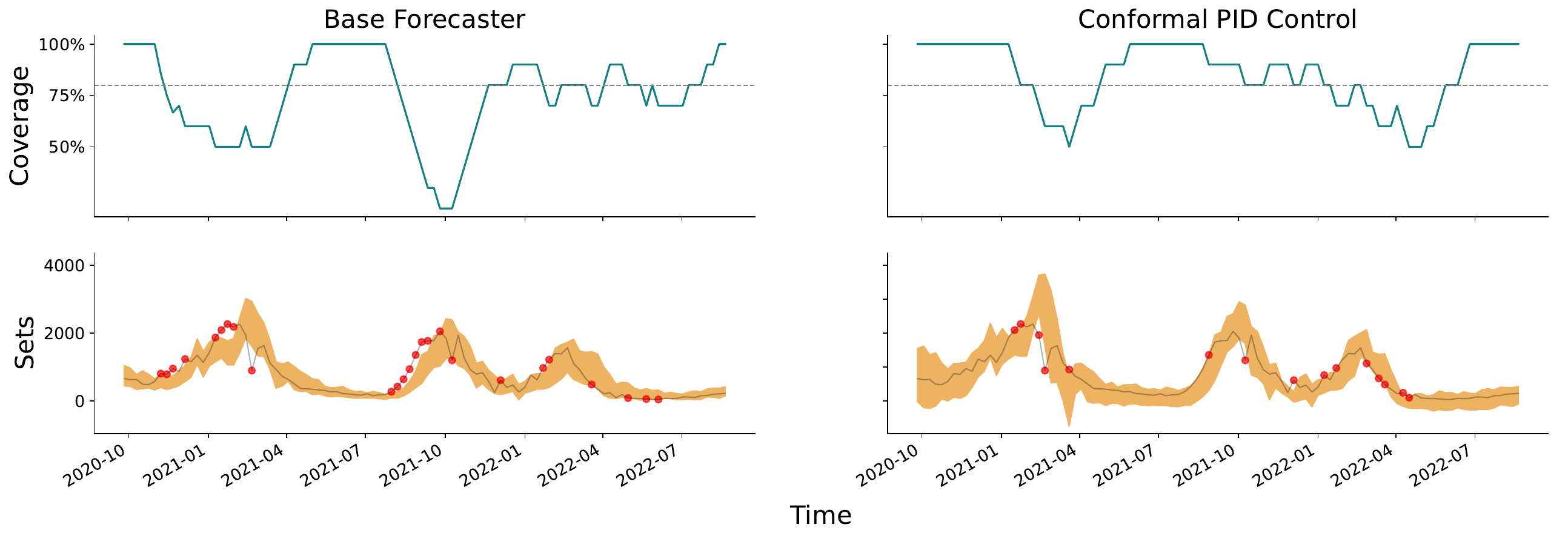}

\smallskip\smallskip
{\footnotesize \begin{tabular}{lll}
\toprule
& Base Forecaster & Conformal PID Ctrl \\
\midrule
Marginal coverage & 0.75 & 0.85 \\
Longest err sequence & 7 & 2 \\
Average set size & 556 & 827 \\
Median set size & 482 & 706 \\
75\% quantile set size & 759 & 1.03e+03 \\
90\% quantile set size & 993 & 1.36e+03 \\
95\% quantile set size & 1.11e+03 & 1.57e+03 \\
\bottomrule
\end{tabular}
}
\caption{Results for 4-week ahead COVID-19 death forecasting in Texas.}
\label{fig:tx}
\end{figure}

\section{Further experiments}
\label{app:further-experiments}

We give a more comprehensive view of our results, examing all data
sets, and a range of tuning parameters for each method. We restrict our
attention to AR as the base forecaster; for the rest of the base forecasters, we
refer to the GitHub repository:
\url{https://github.com/aangelopoulos/conformal-time-series}.  

For each experiment, we describe the data set in a new subsection, and two plots  
are included: one for the coverage, and one for the prediction sets. Each column
in the plots represents a different method, and each row is a different learning
rate. For the quantile tracker, the learning rate is to be interpreted 
as the multiplier in front of \smash{$\hat{B}_t$}. Each method is given a
different color, which stays consistent throughout the plots. We use a tan
integrator and a Theta scorecaster throughout, just as in the main text 
experiments. 

\subsection{Amazon/Google}

These data sets are part of a multivariate time series consisting of thirty
blue-chip stock prices, including those of Amazon (AMZN) and Google (GOOGL),
from January 1, 2006 to December 31, 2014. We attempt to forecast the daily 
opening price of each of Amazon and Google stock, on a log scale. Available to
the scorecaster are the previous open prices of \emph{all 30 stocks}. 

\subsection{Microsoft}

This data set is a univariate time series consisting of a single stock open
price, that of Microsoft (MSFT), from April 1, 2015 to May 31, 2021. 

\subsection{Daily temperature in Delhi}

This data set contains the daily temperature (averaged over 8 measurements in 3
hour periods), humidity, wind speed, and atmospheric temperature in the city of
Delhi from January 1, 2003 to April 24, 2017, scraped using the Weather
Underground API. 

\subsection{Electricity demand forecasting}

This data set measures electricity demand in New South Wales collected at
half-hour increments from May 7th, 1996 to December 5th, 1998 (we zoom in on the
first 2000 time points). There are also several other variables collected, such
as the demand and price in Victoria, the amount of energy transfer between New
South Wales and Victoria, and so on. These are given as covariates to the
scorecaster. The demand value is normalized by default to lie in $[0,1]$. 

\subsection{Synthetic data sets}

We perform some experiments on two synthetic score sequences which include
change points and other behaviors difficult to produce using real data. In this
setting, there is no ground truth $y_t$ sequence, so we do not plot the sets. 
Instead, we plot the scores themselves in one column, and the quantiles $q_t$
produced by each algorithm in a different column (when $q_t \geq s_t$, we
cover). The general goal is for $q_t$ to track the $1-\alpha$ quantile of $s_t$,
and if it is too far off, that corresponds to the ``set being too large or too
small'' in a situation where we would be constructing sets out of these scores.  

We consider an i.i.d.\ sequence of scores, a noisy increasing sequence of
scores, and a mix of change points and trends. Our codebase describes the score
generation procedure in more detail. 

\begin{figure}[p]
\includegraphics[width=\textwidth]{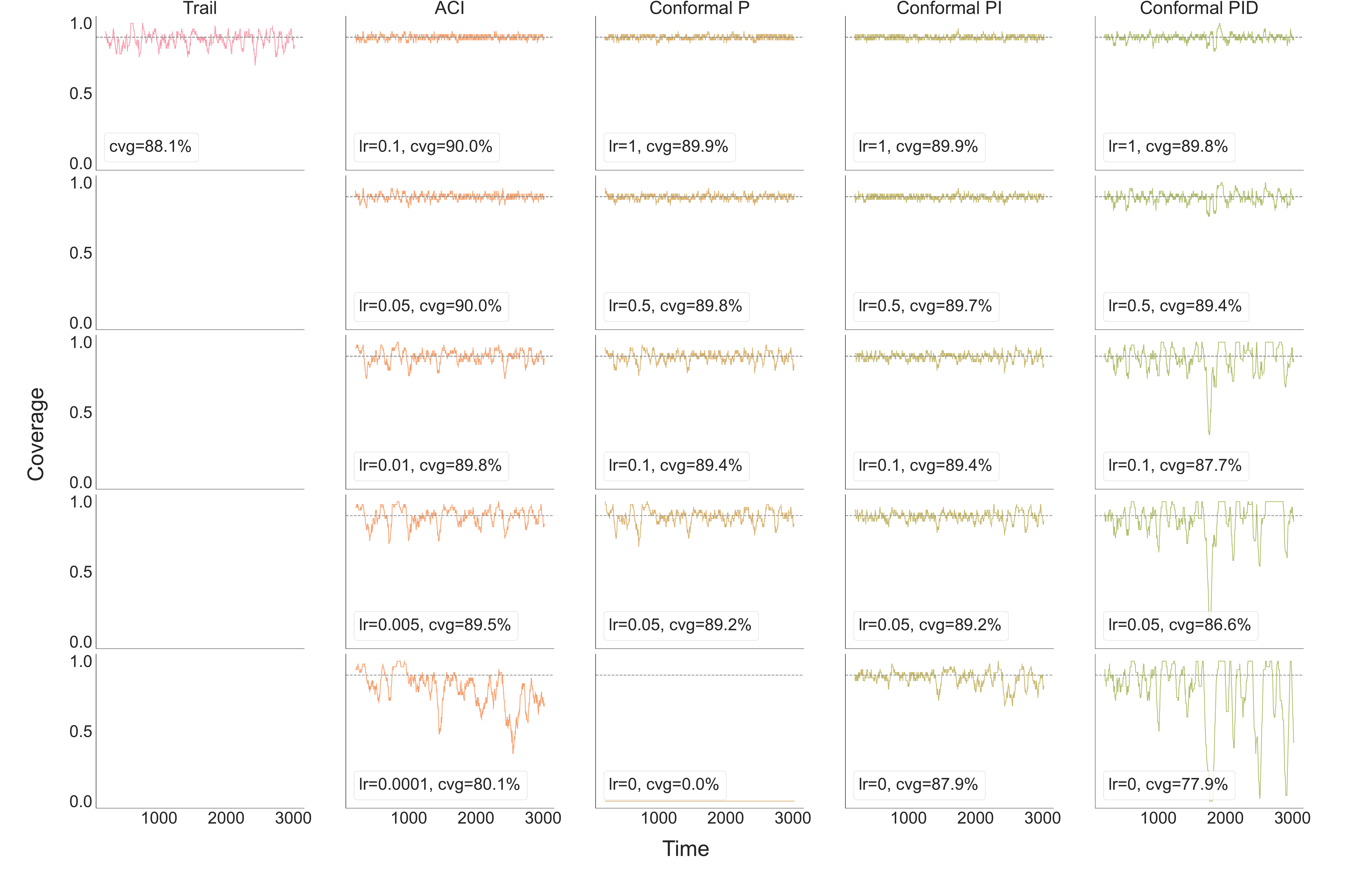} 

\bigskip\bigskip
\includegraphics[width=\textwidth]{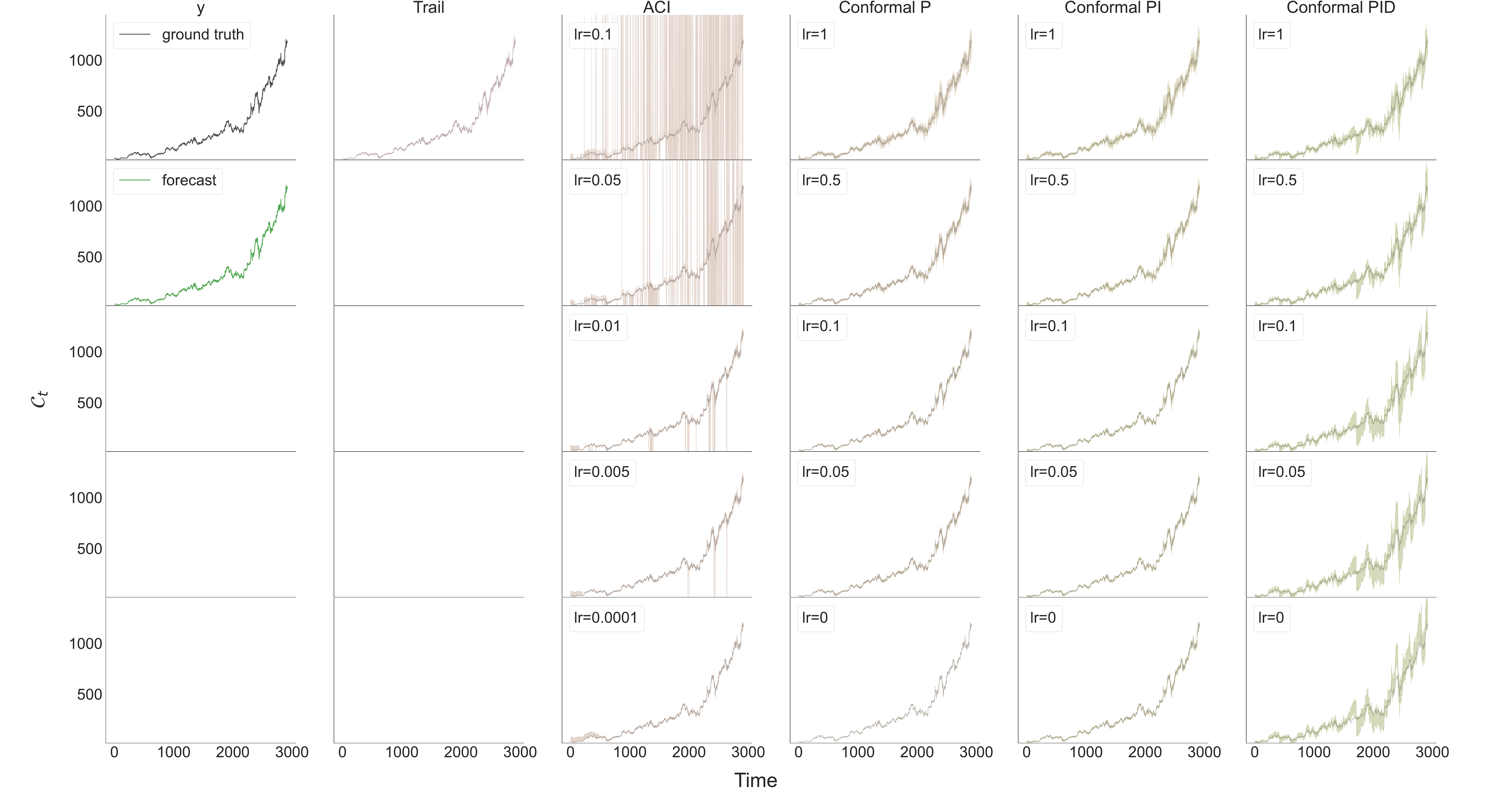}
\caption{Results for the Amazon data set.}
\label{fig:AMZN-appendix}
\end{figure}

\begin{figure}[p]
\includegraphics[width=\textwidth]{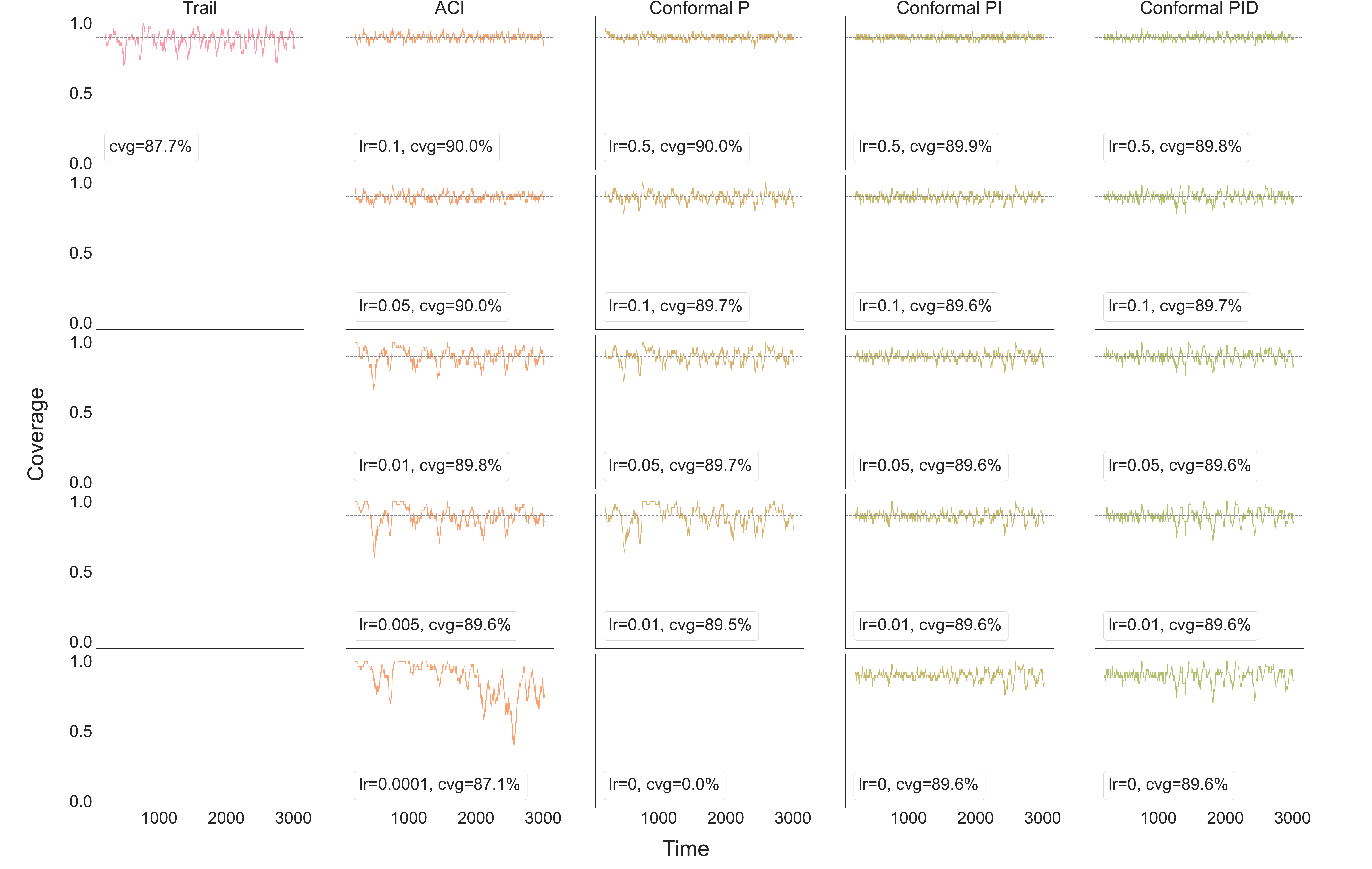} 

\bigskip\bigskip
\includegraphics[width=\textwidth]{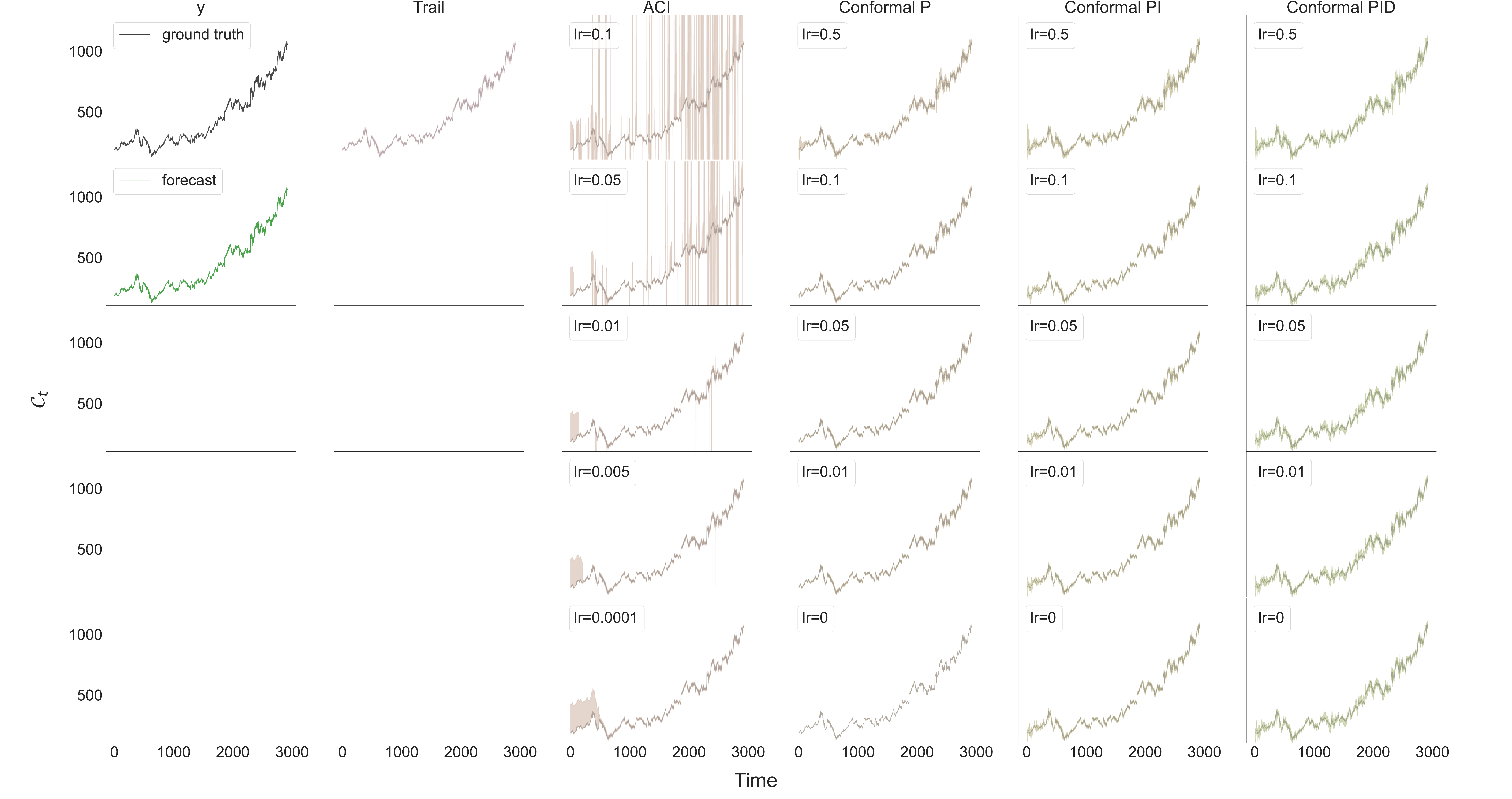}
\caption{Results for the Google data set.}
\label{fig:GOOGL-appendix}
\end{figure}

\begin{figure}[p]
\includegraphics[width=\textwidth]{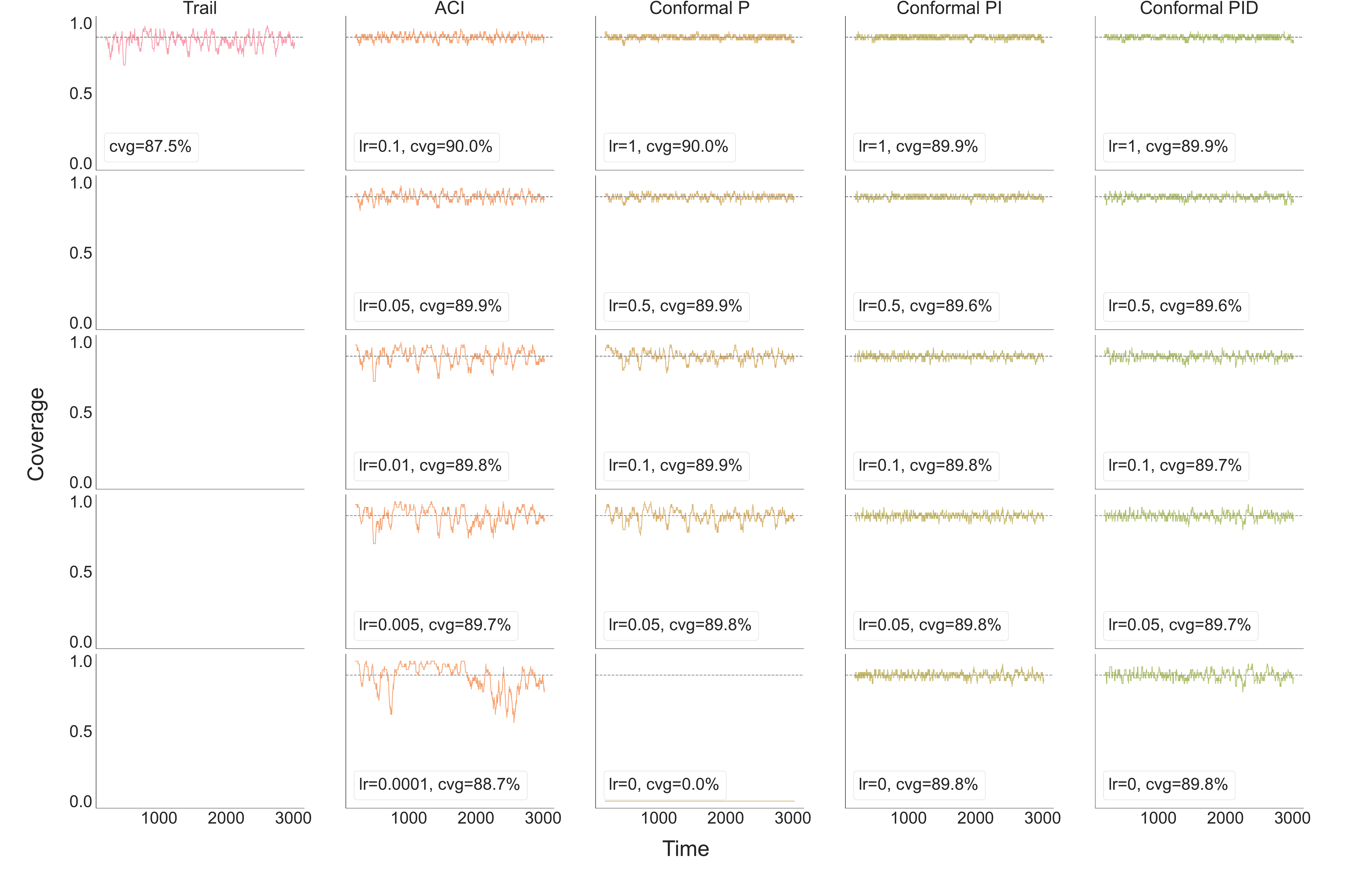} 

\bigskip\bigskip
\includegraphics[width=\textwidth]{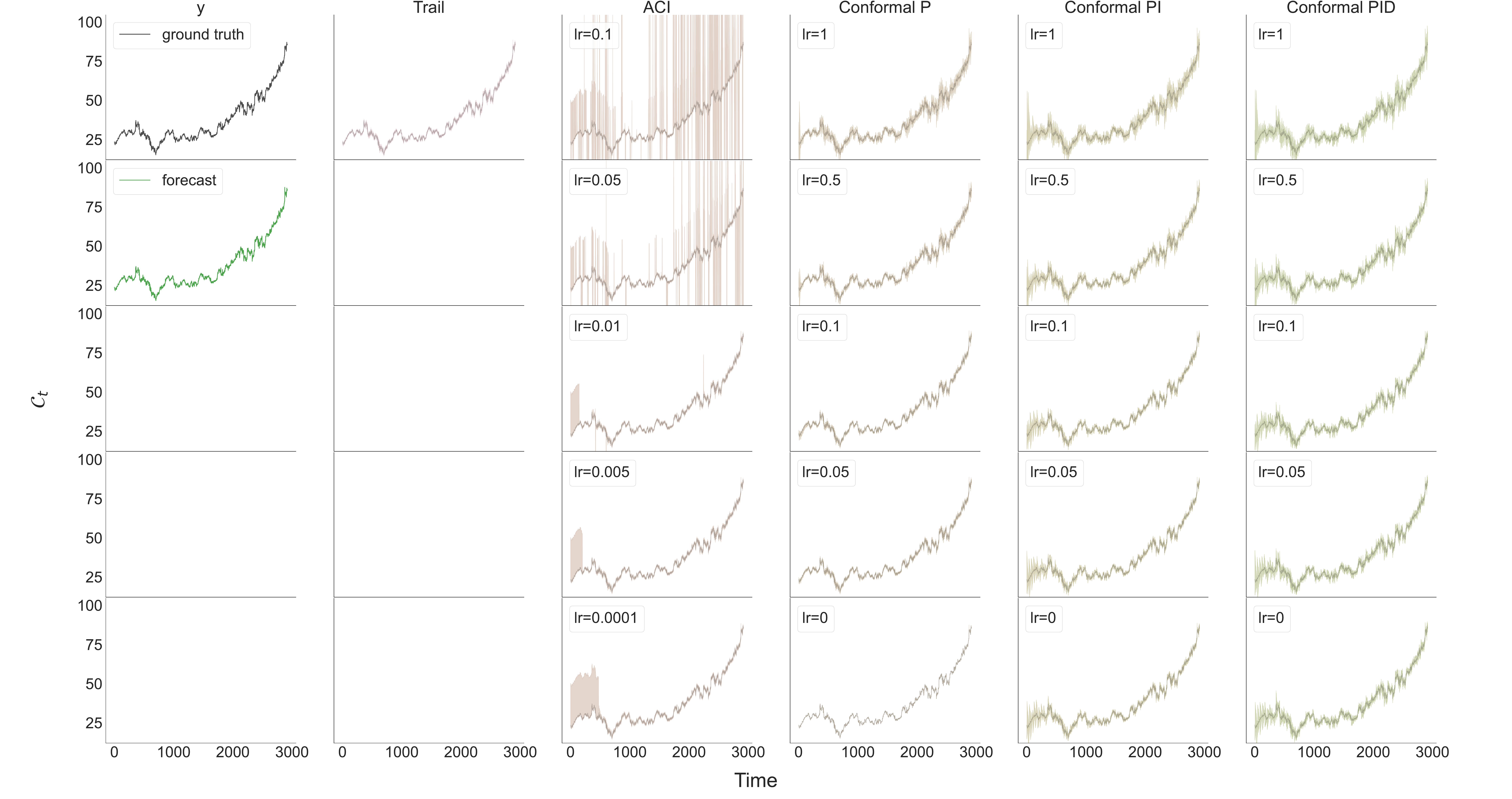}
\caption{Results for the Microsoft data set.}
\label{fig:MFST-appendix}
\end{figure}

\begin{figure}[p]
\includegraphics[width=\textwidth]{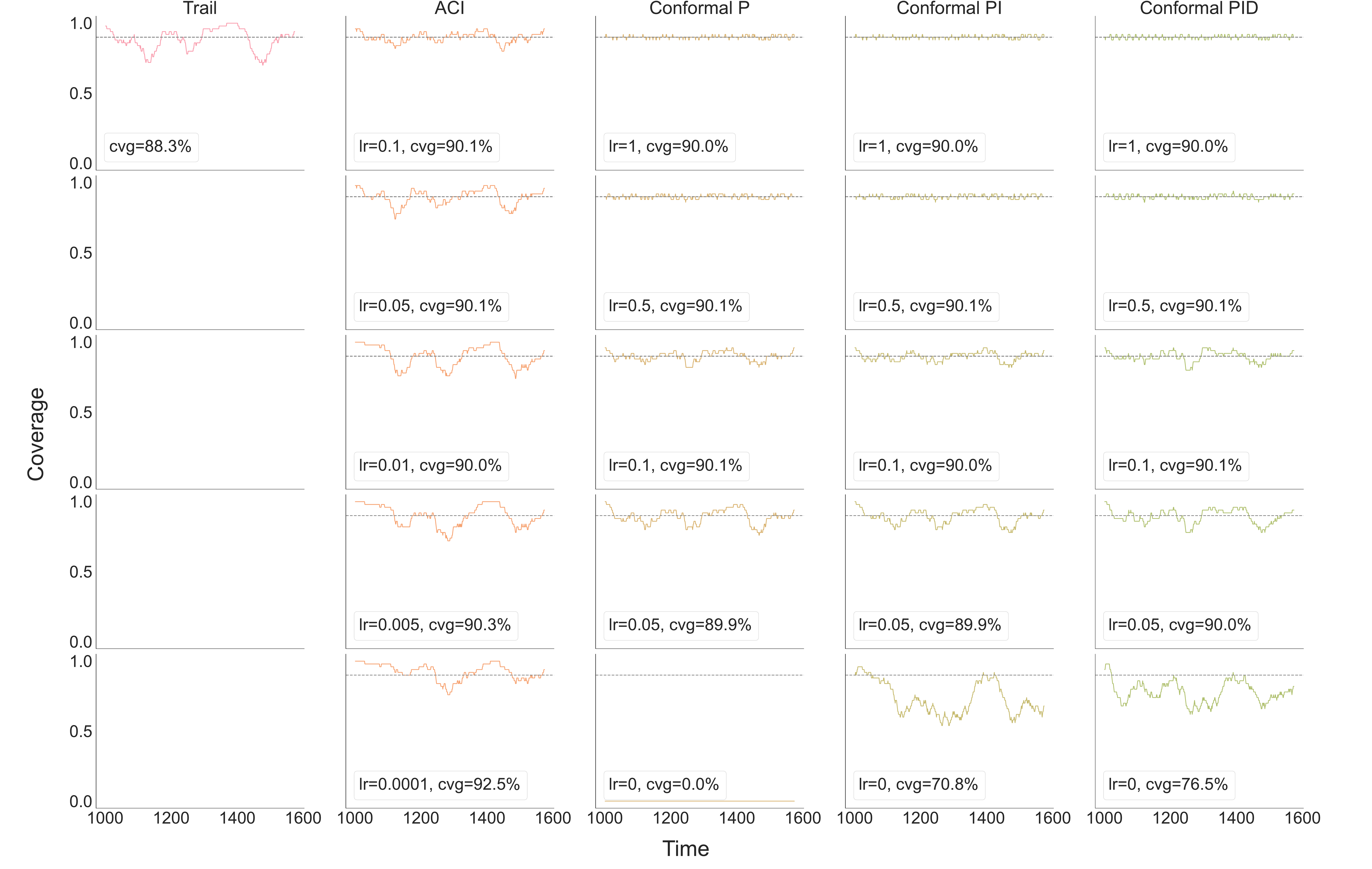} 

\bigskip\bigskip
\includegraphics[width=\textwidth]{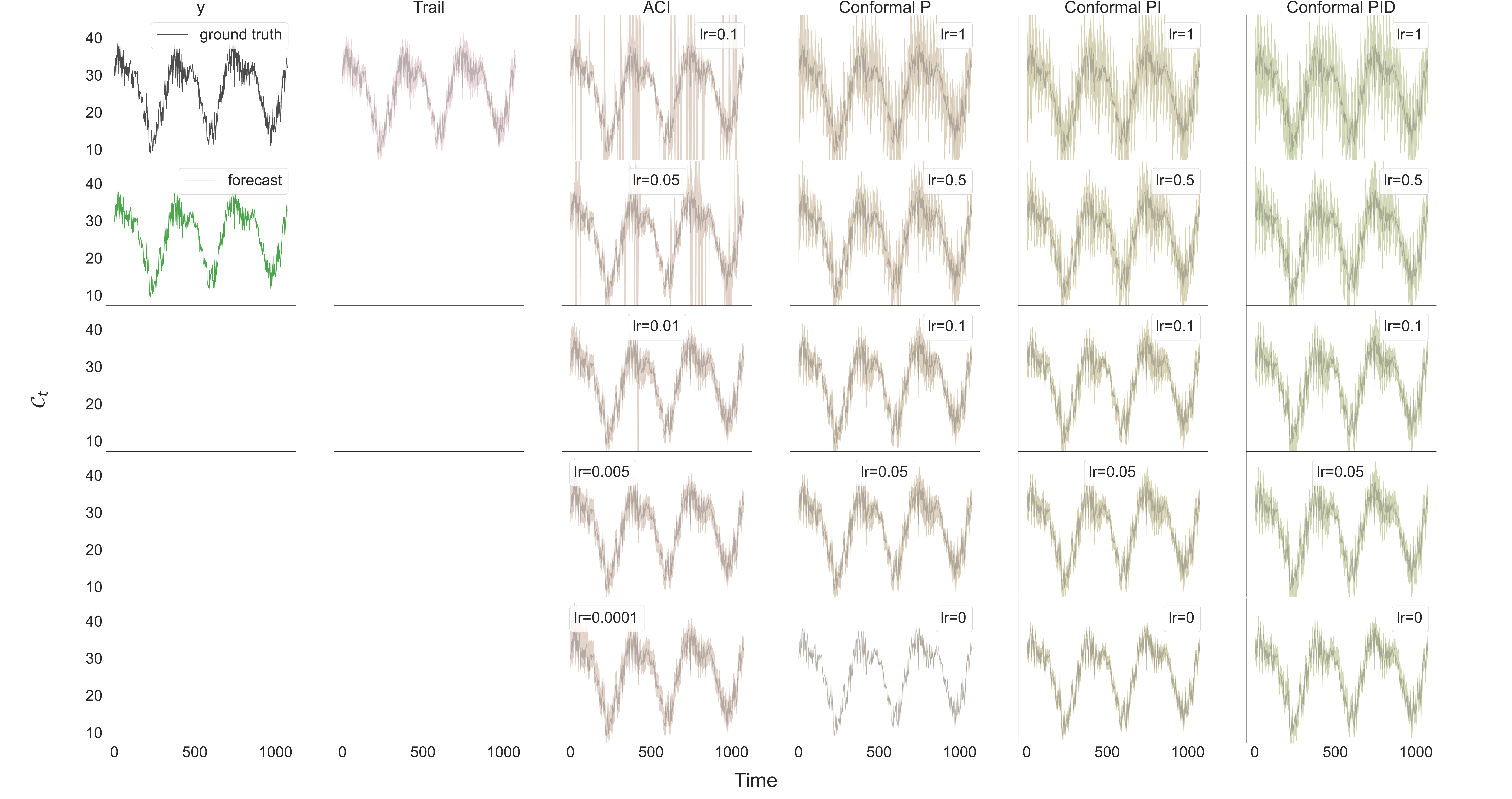}
\caption{Results for the Delhi temperature data set.}
\label{fig:daily-climate-appendix}
\end{figure}

\begin{figure}[p]
\includegraphics[width=\textwidth]{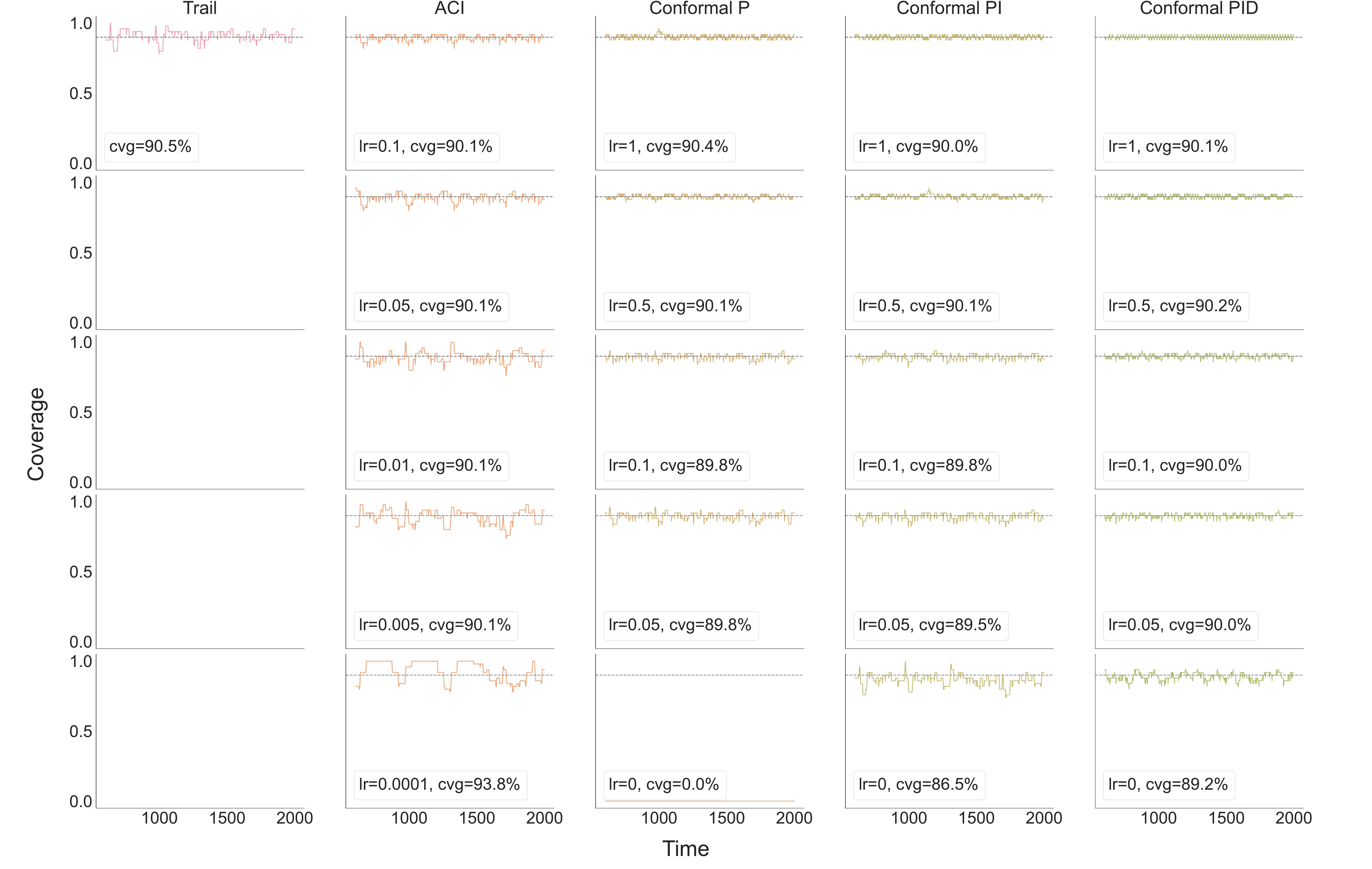} 

\bigskip\bigskip
\includegraphics[width=\textwidth]{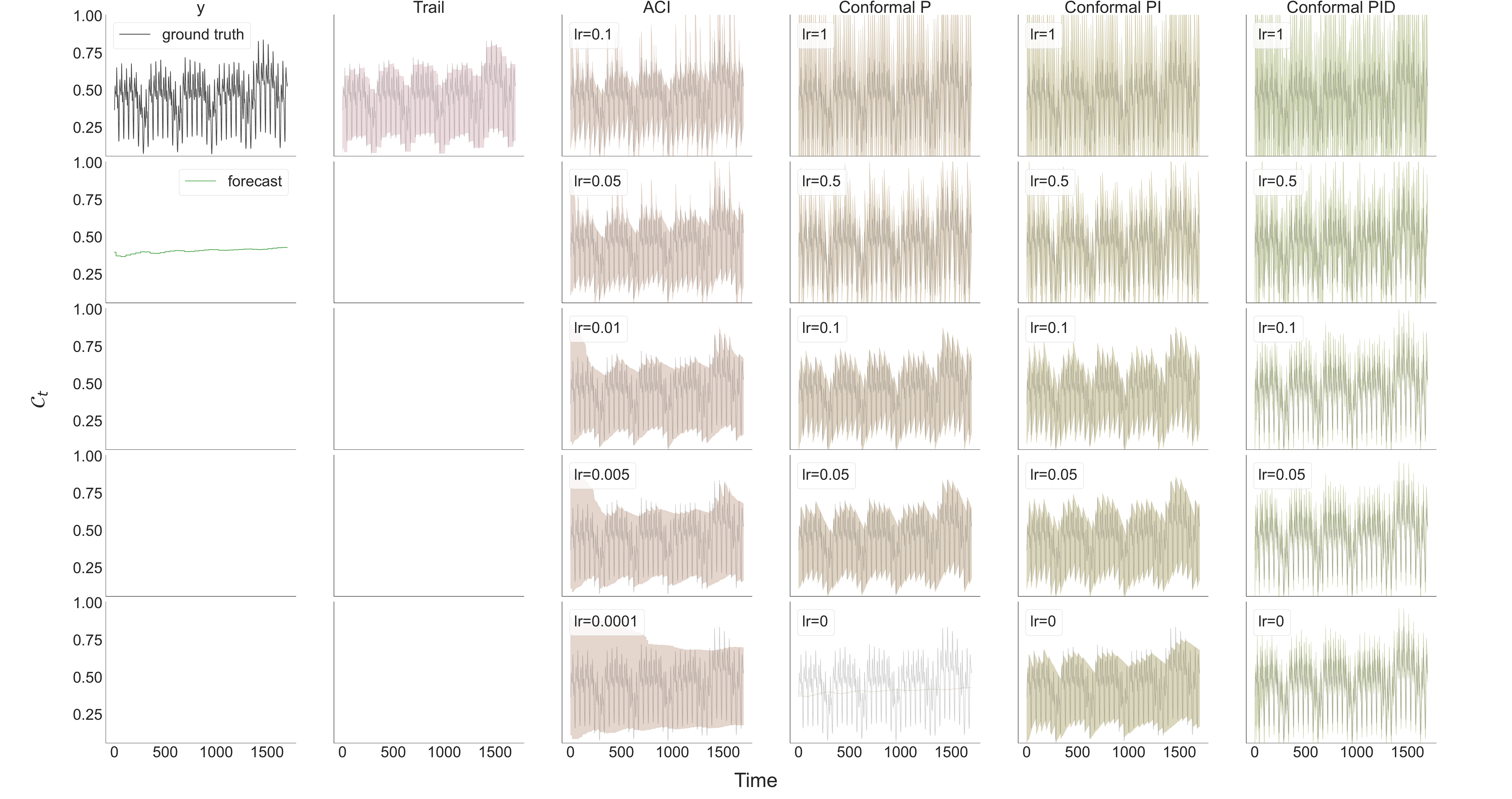}
\caption{Results for the electricity demand forecasting data set.}
\label{fig:elec2-appendix}
\end{figure}

\begin{figure}[p]
\includegraphics[width=\textwidth]{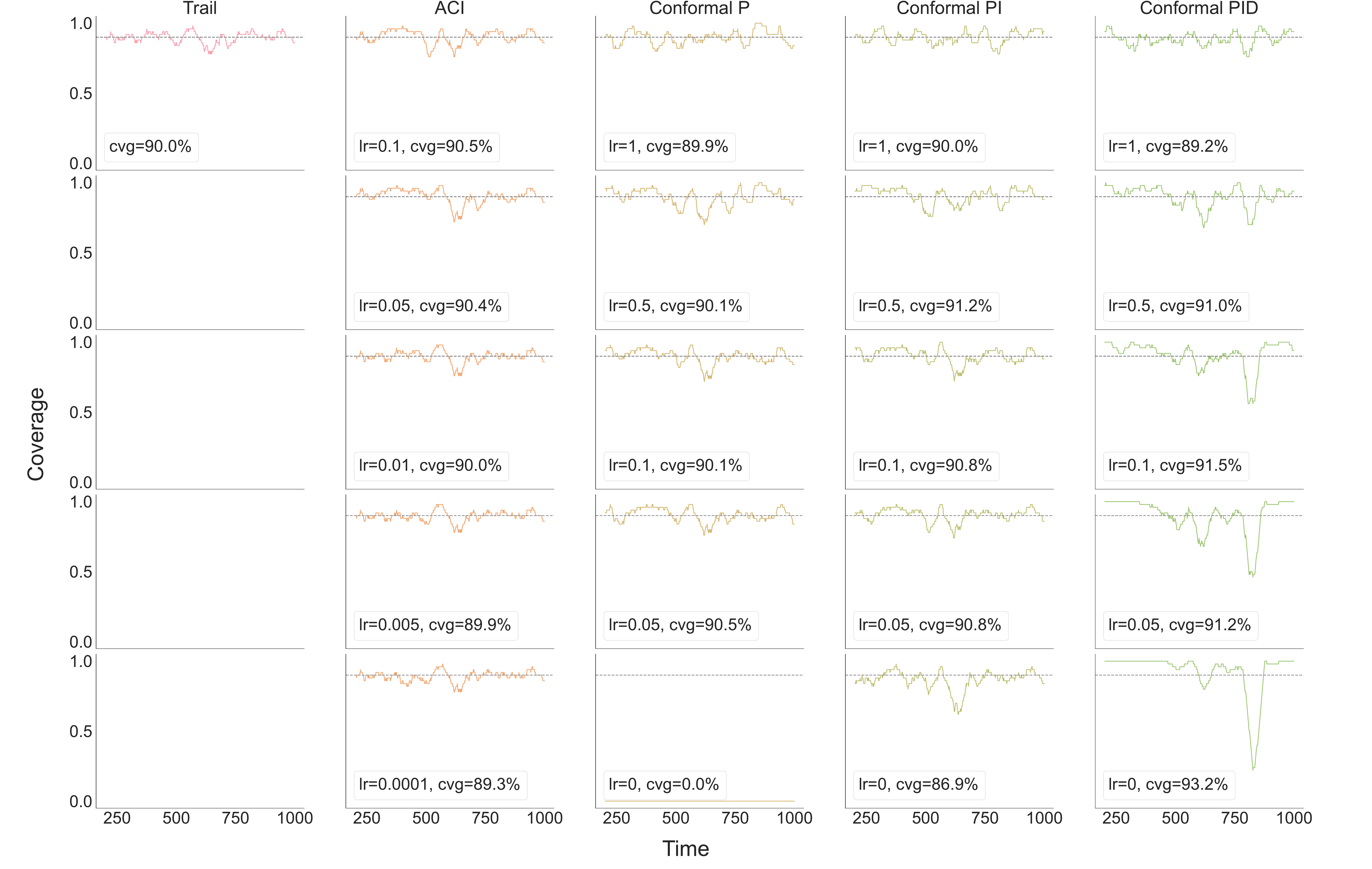} 

\bigskip\bigskip
\includegraphics[width=\textwidth]{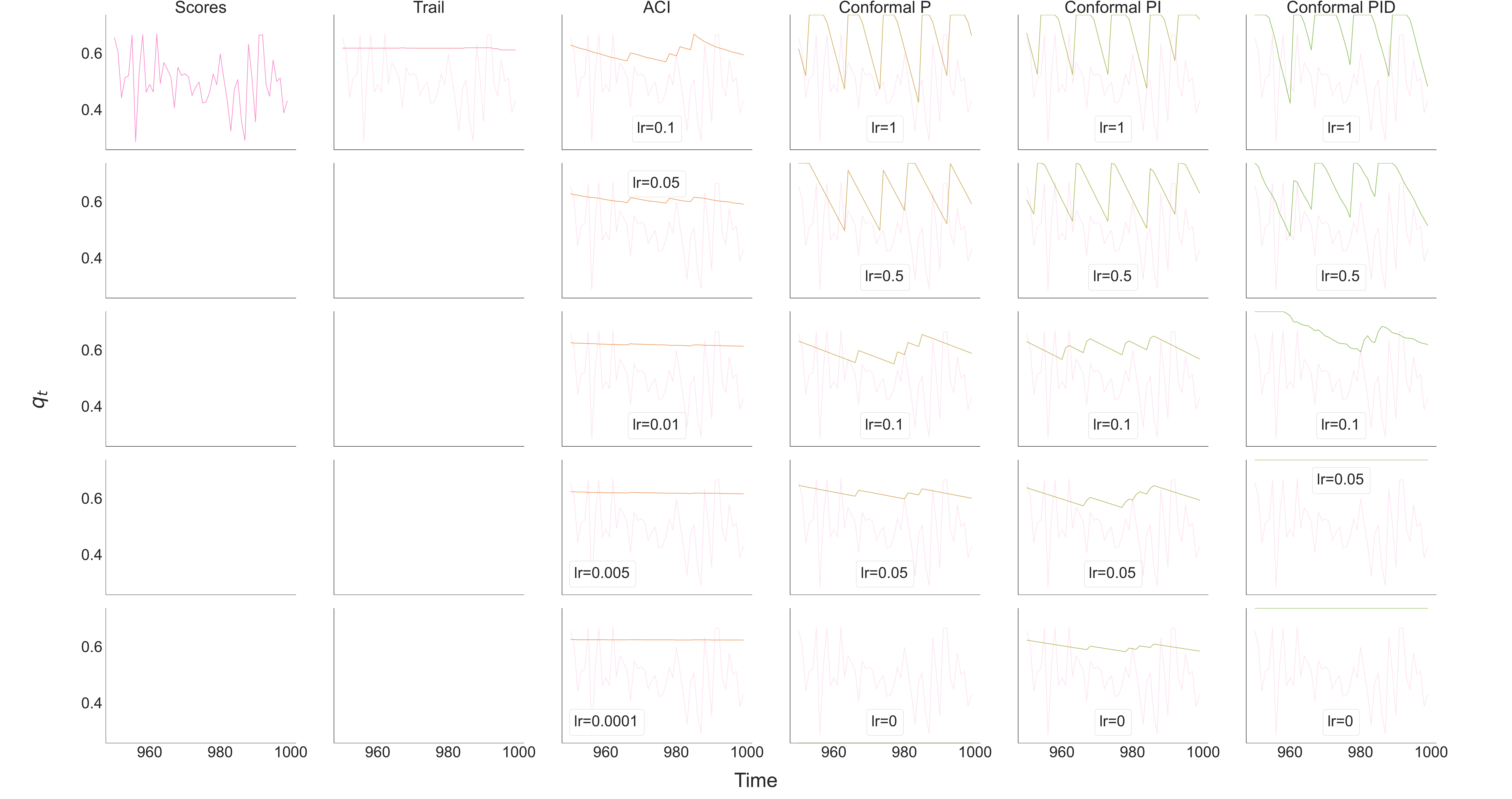}
\caption{Results for an i.i.d.\ score sequence.}
\label{fig:stationary-appendix}
\end{figure}

\begin{figure}[p]
\includegraphics[width=\textwidth]{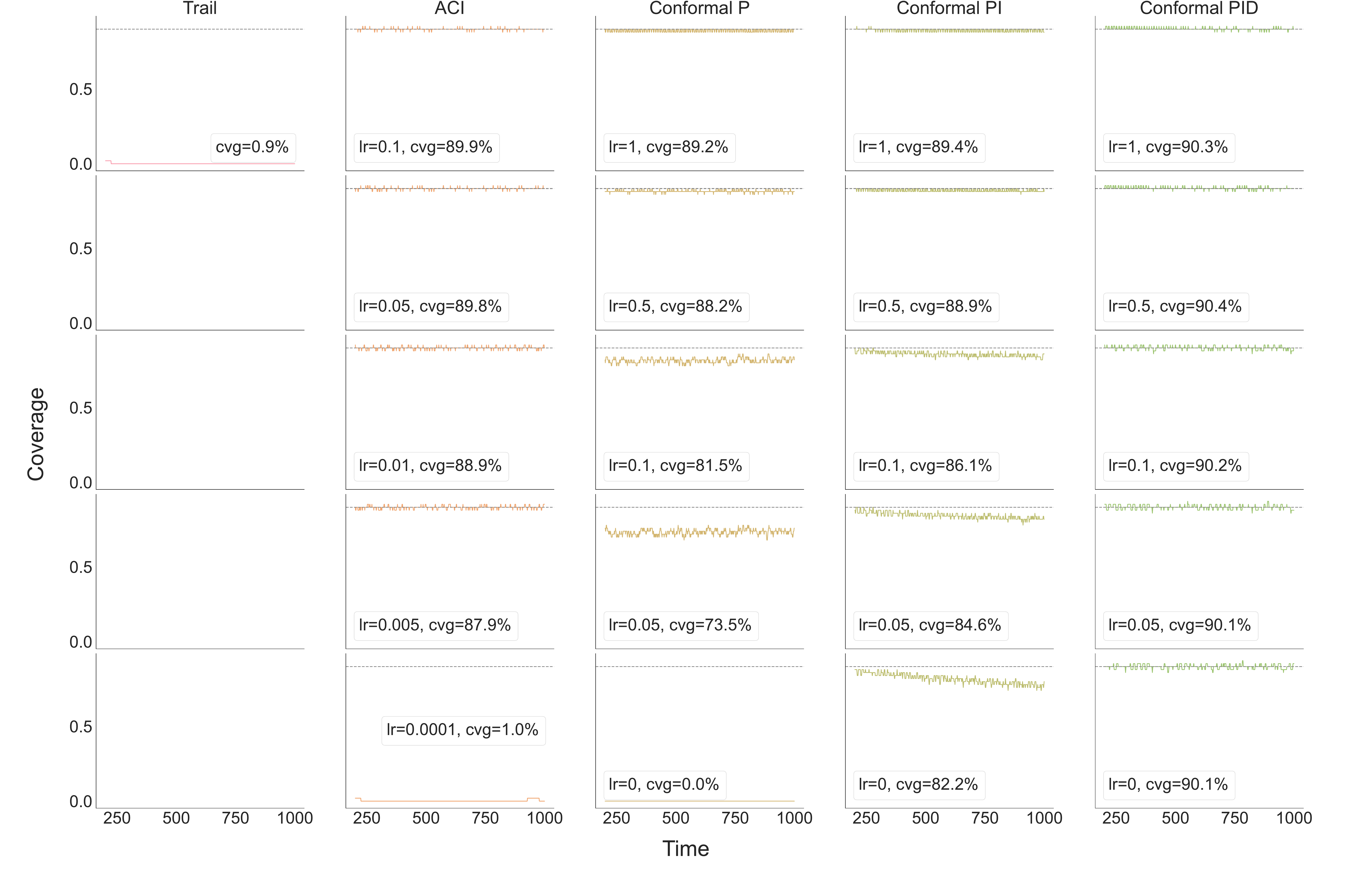} 

\bigskip\bigskip
\includegraphics[width=\textwidth]{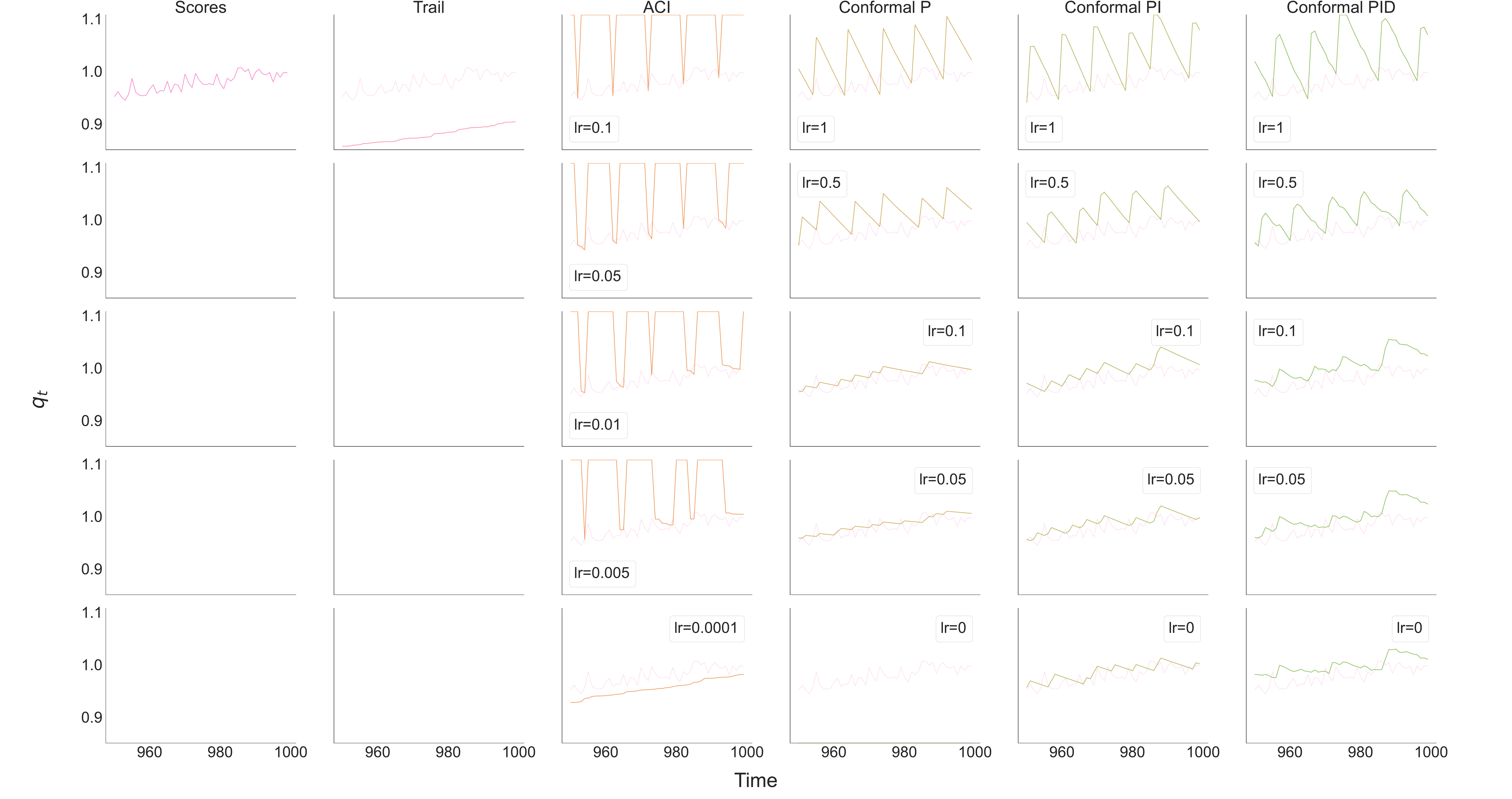}
\caption{Results for an increasing score sequence.}
\label{fig:increasing-appendix}
\end{figure}

\begin{figure}[p]
\includegraphics[width=\textwidth]{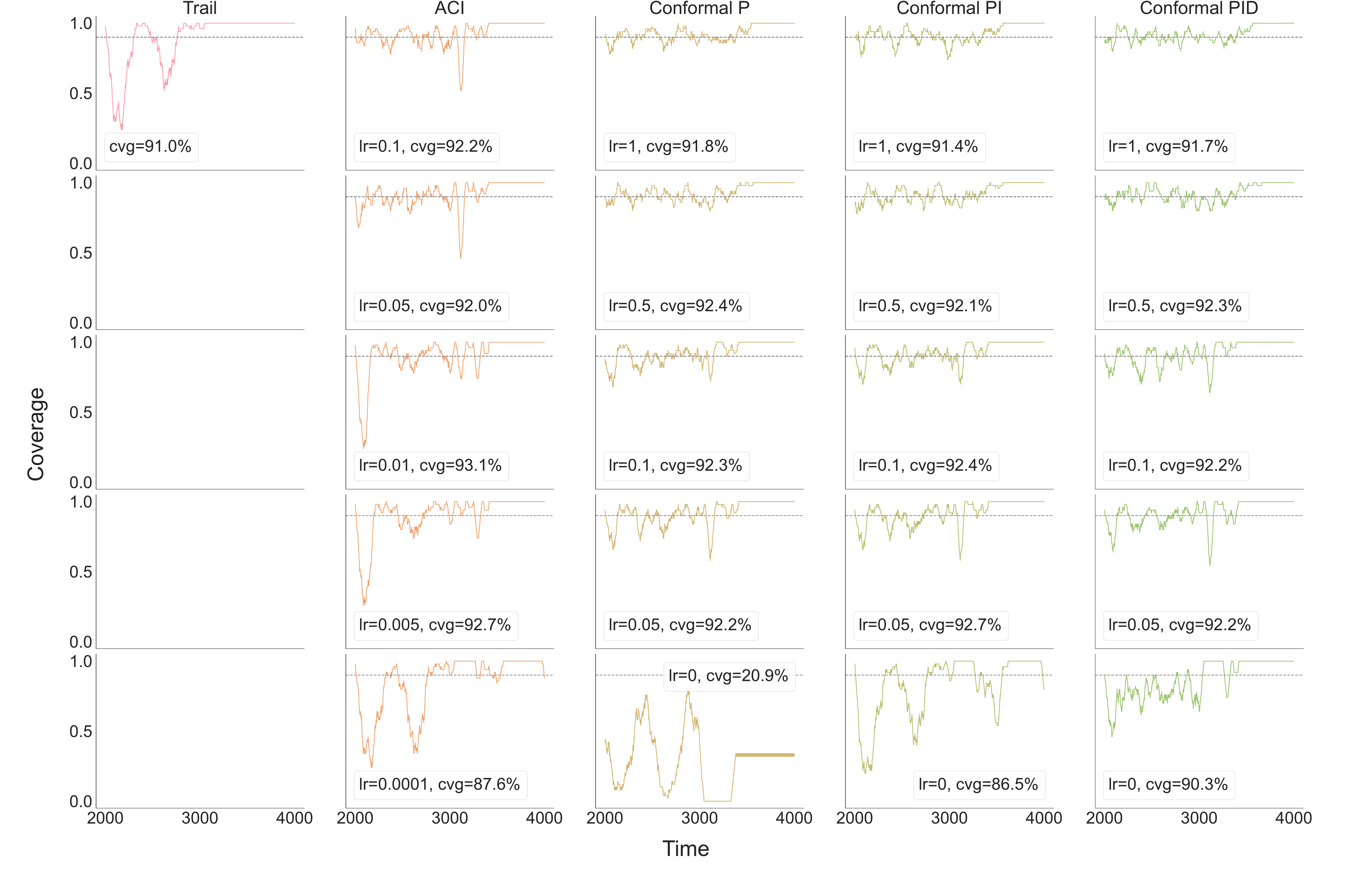} 

\bigskip\bigskip
\includegraphics[width=\textwidth]{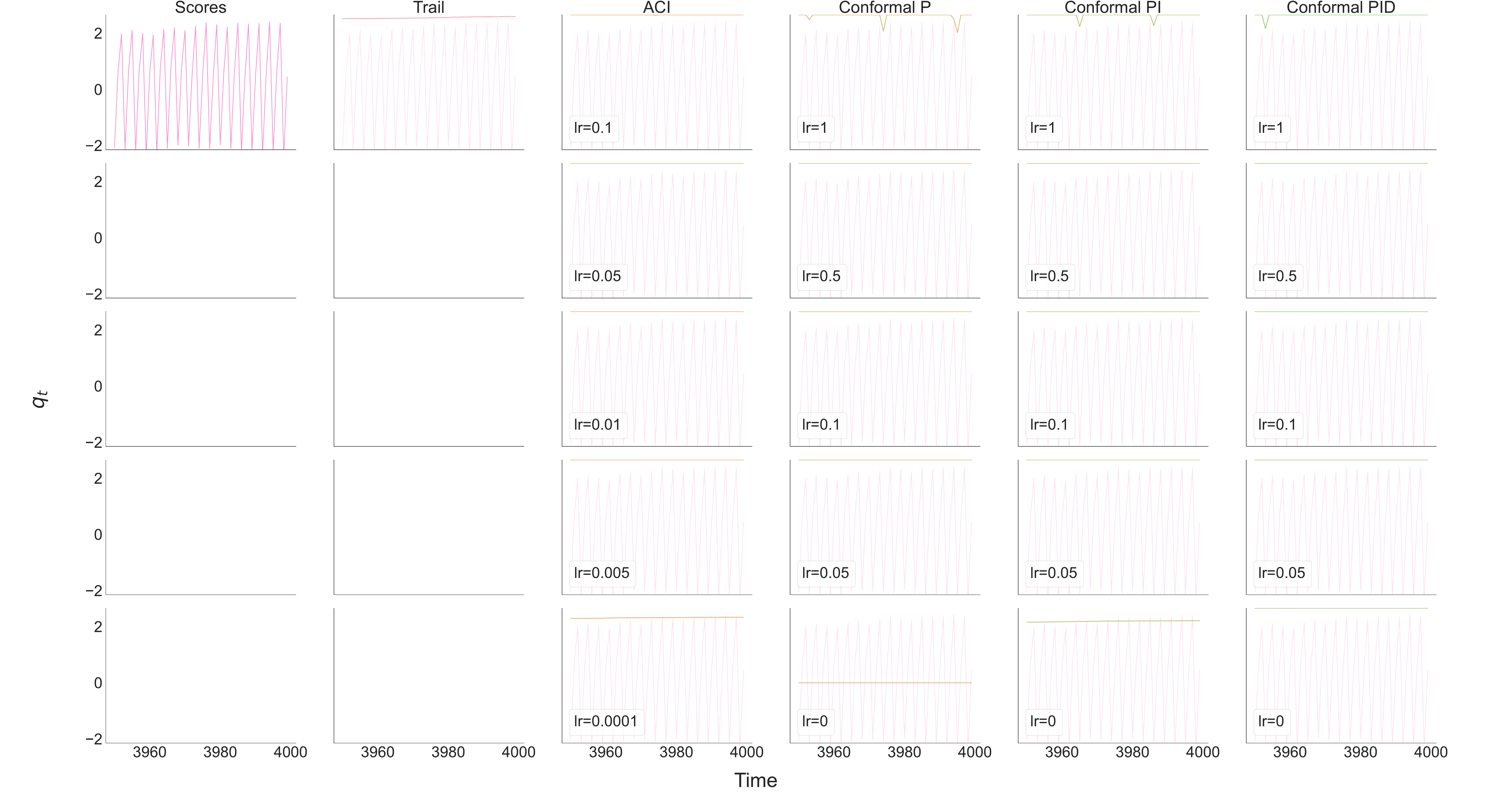}
\caption{Results for a score sequence that is a mix of change points and trends.}
\label{fig:mix-appendix}
\end{figure}

\end{document}